\pdfoutput=1
\documentclass{article} 
\usepackage{times}
\usepackage{geometry}

\usepackage[utf8]{inputenc} 
\usepackage[T1]{fontenc}    
\usepackage[colorlinks,citecolor=blue,linkcolor=blue]{hyperref}       
\usepackage{url}            
\usepackage{booktabs}       
\usepackage{amsfonts}       
\usepackage{nicefrac}       
\usepackage{microtype}      

\usepackage{mkolar_definitions}
\usepackage{url}
\usepackage{authblk}

\usepackage{graphicx}
\usepackage{subfigure}
\usepackage{multirow}
\usepackage{floatrow}
\newfloatcommand{capbtabbox}{table}[][\FBwidth]

\PassOptionsToPackage{numbers}{natbib}
\usepackage{natbib}

\usepackage{algorithm}
\usepackage{algorithmicx,algpseudocode}
\usepackage{tikz}
\usepackage{xspace}
\usepackage{graphicx}

\usepackage{enumerate}
\usepackage{amsthm,amsmath,amssymb}
\usepackage{amsfonts,dsfont}
\usepackage{nicefrac}
\usepackage{microtype}
\usepackage{mathtools}
\usepackage{footnote}

\usepackage{comment}
\newcount\Comments  
\definecolor{darkgreen}{rgb}{0,0.5,0}
\definecolor{darkred}{rgb}{0.7,0,0}
\definecolor{teal}{rgb}{0.3,0.8,0.8}
\newcommand{\kibitz}[2]{\ifnum\Comments=1{\textcolor{#1}{\textsf{\footnotesize #2}}}\fi}

\newcommand{\version}{arxiv}

\newenvironment{packed_enum}{
  \begin{enumerate}
    \setlength{\itemsep}{1pt}
    \setlength{\parskip}{-1pt}
    \setlength{\parsep}{0pt}
}{\end{enumerate}}

\newcounter{qcounter}
 {\end{list}}

\newcommand{\oracle}{\ensuremath{\textsc{Oracle}}}
\newcommand{\topk}{\ensuremath{\textsc{Top-k}}}
\newcommand{\biclique}{\ensuremath{\textsc{Biclique}}}
\newcommand{\matching}{\ensuremath{\textsc{Matching}}}

\newcommand{\matroid}{\ensuremath{\textsc{Matroid}}}
\newcommand{\Low}{\ensuremath{\textrm{Low}}}
\newcommand{\Dis}{\ensuremath{\textsc{Disagree}}}
\newcommand{\true}{\ensuremath{\textsc{True}}}
\newcommand{\false}{\ensuremath{\textsc{False}}}

\newcommand{\tlog}{\ensuremath{\widetilde{\log}}}
\newcommand{\conv}{\textrm{conv}}

\begin{document} 

\title{Disagreement-Based Combinatorial Pure Exploration: Sample Complexity Bounds and an Efficient Algorithm}

\author[1]{
Tongyi Cao%
}
\author[2]{
Akshay Krishnamurthy%
\thanks{\{tcao,akshay\}@cs.umass.edu}}

\affil[1]{University of Massachusetts, Amherst, MA}
\affil[2]{Microsoft Research, New York, NY}
\maketitle

\begin{abstract}
  We design new algorithms for the combinatorial pure exploration
  problem in the multi-arm bandit framework. In this problem, we are
  given $K$ distributions and a collection of subsets $\Vcal \subset
  2^{[K]}$ of these distributions, and we would like to find the
  subset $v \in \Vcal$ that has largest mean, while collecting, in a
  sequential fashion, as few samples from the distributions as
  possible. In both the fixed budget and fixed confidence settings,
  our algorithms achieve new sample-complexity bounds that provide
  polynomial improvements on previous results in some settings. Via an
  information-theoretic lower bound, we show that no approach based on
  uniform sampling can improve on ours in any regime, yielding the
  first interactive algorithms for this problem with this basic
  property.  Computationally, we show how to efficiently implement our
  fixed confidence algorithm whenever $\Vcal$ supports efficient
  linear optimization.  Our results involve precise
  concentration-of-measure arguments and a new algorithm for linear
  programming with exponentially many constraints.
\end{abstract}

\section{Introduction}
\label{sec:intro}


Driven by applications in engineering and the sciences, much
contemporary research in mathematical statistics focuses on recovering
structural information from noisy data.  Combinatorial structures that
have seen intense theoretical investigation include
clusterings~\citep{mossel2014belief,abbe2016exact,balakrishnan2011noise},
submatrices~\citep{butucea2013detection,kolar2011minimax,chen2014statistical},
and graph theoretic structures like matchings, spanning trees, and
paths~\citep{ariascastro2008searching,addarioberry2010combinatorial}.
In this paper, we design interactive learning algorithms for these
structure discovery problems.


Our mathematical formulation is through the \emph{combinatorial pure
  exploration for multi-armed bandits}
framework~\citep{chen2014combinatorial}, a recent generalization of
the best-arm identification problem~\citep{mannor2004sample,audibert2010best}. 
In this setting, we are given a combinatorial decision set
$\Vcal \subset 2^{[K]}$ and access to $K$ arms, where each arm $a \in [K]$
is associated with a distribution with unknown mean $\mu_a$. We can, in
sequential fashion, query an arm and obtain an iid sample from the
corresponding distribution, and the goal is to identify the subset
$v \in \Vcal$ with maximum mean $\mu(v) = \sum_{a \in v} \mu_a$ while
minimizing the number of samples collected.

This model has been studied in recent work both in the general form
and with specific decision sets $\Vcal$.  For specific structures, a
line of work established near-optimal algorithms for any $\Vcal$ that
corresponds to the bases of a
matroid~\citep{kalyanakrishnan2012pac,kaufmann2013information,chen2016pure}
and slight generalizations~\citep{chen2014combinatorial}.
After~\citet{chen2014combinatorial} introduced the general problem,
\citet{gabillon2016improved} and \citet{chen2017nearly} made
interesting progress with improved guarantees reflecting precise
dependence on the underlying mean vector $\mu$.  However, these
results fail to capture intricate combinatorial structure of the
decision set, and, as we show, they can be polynomially worse than a
simple non-interactive algorithm based on maximum likelihood
estimation. With this in mind, our goal is to capture this
combinatorial structure to design an algorithm that is never worse
than the non-interactive baseline, but that can be much better.

Since we are doing combinatorial optimization, we typically consider
decision sets $\Vcal$ that are exponentially large but have small
description length, so that direct enumeration of the elements in
$\Vcal$ is not computationally tractable. Instead, we assume that
$\Vcal$ supports efficient linear optimization, and our main algorithm
only accesses $\Vcal$ through a linear optimization oracle. To shed
further light on purely statistical issues, we also present some
results for computationally inefficient algorithms.

\paragraph{Our Contributions.}
We make the following contributions:
\begin{packed_enum} 
\item First, we derive the minimax optimal sample complexity in the
  non-interactive setting, where arms are queried uniformly. This
  precisely characterizes how the structure of $\Vcal$ influences the
  sample complexity and also provides a baseline for evaluating
  interactive algorithms.
\item In the fixed confidence setting, we design two algorithms that
  are never worse than the non-interactive minimax rate, but that can
  adapt to heterogeneity in the problem to be substantially better. On
  the computational side, we show how to implement the first algorithm
  in polynomial time with access to a linear optimization oracle. The
  second algorithm is computationally inefficient, but has a strictly
  better sample complexity.
\item In the fixed budget setting, we design an algorithm with similar
  statistical improvements, improving on the MLE when there is
  heterogeneity in the problem.
\item We perform a careful comparison to prior work, with several
  concrete examples. We show that prior
  results~\citep{chen2014combinatorial,gabillon2016improved,chen2017nearly}
  can be polynomially worse than the non-interactive minimax rate,
  which contrasts with our guarantees. We also describe other settings where
  our results outperform prior work, and vice versa.
\end{packed_enum}

\paragraph{Our Techniques.}

The core of our statistical analysis is a new deviation bound for
combinatorial pure exploration that we call a \emph{normalized regret
  inequality}. We prove that to recover the optimal subset $v^\star$,
it suffices to control, for each $v \in \Vcal$, the sampling error in
the mean difference between $v$ and $v^\star$ at a level proportional
to the symmetric set difference between the two.\footnote{The name
  arises because the inequality involves comparison with the optimum
  and is normalized by the set difference.} In the non-interactive
setting, the normalized regret inequality always yields the optimal
sample complexity (as we prove in~\pref{thm:noninteractive_upper}),
and is often sharper than more standard uniform convergence arguments
(e.g., over the arms or the decision set) that have been used in
prior work. Our new guarantees stem from using this new inequality in
interactive procedures.

The fixed confidence setting poses a significant challenge, since
confidence bounds typically appear algorithmically, but the normalized
regret inequality is centered around the optimum, which is of course
unknown! We address this difficulty with an elimination-style
algorithm that eliminates a hypothesis $v \in \Vcal$ when any other
candidate is significantly better and that queries only where the
survivors disagree. Using only the normalized regret inequality, we
can prove that $v^\star$ is never eliminated, but also that $v^\star$
will eventually eliminate every other hypothesis.  This algorithm
resembles approaches for disagreement-based active
learning~\citep{hanneke2014theory}, but uses a much stronger
elimination criteria that is crucial for obtaining our sample
complexity guarantees.

Computationally, deciding if the surviving candidates disagree on an
arm poses further challenges, since the description of the surviving
set involves exponentially many constraints, one for each candidate
$u \in \Vcal$. This problem can be written as a linear program, which
we can solve using the Plotkin-Shmoys-Tardos reduction to online
learning~\citep{plotkin1995fast}. However, since there are
exponentially many constraints in the LP, the standard approach of
using multiplicative weight updates fails, but, exploiting further
structure in the problem, we can run Follow-the-Perturbed-Leader~\citep{kalai2005efficient},
since the online learner's problem is actually linear in the
candidates $u$ that parameterize the constraints. Thus with a linear
optimization oracle, we obtain an efficient algorithm
for the fixed confidence
setting.

\section{Preliminaries}
\label{sec:prelims}
In the combinatorial pure exploration problem, we are given a finite
set of arms $A \triangleq \{1,\ldots, K\}$, where arm $a$ is
associated with a sub-Gaussian distribution $\nu_a$ with unknown mean
$\mu_a \in [-1,1]$ and variance parameter $1$.\footnote{Recall a
  sub-Gaussian random variable $X$ with mean $\mu$ and variance parameter $\sigma^2$
  satisfies $\EE[\exp(s(X-\mu))] \leq \exp(\sigma^2 s^2/2)$. The results
  easily generalize to arbitrary known variance parameter.}  Further, we
are given a decision set $\Vcal \subseteq 2^A$.  For $u, v \in \Vcal$,
we use $d(u,v) \triangleq |u \ominus v|$ where $\ominus$ denotes the
symmetric set difference. The goal is to identify a set $v \in \Vcal$
that has the largest collective mean $\sum_{a \in v} \mu_a$.
Throughout the paper, we use the vectorized notation $\mu \triangleq
(\mu_1,\ldots,\mu_K)$ and $\Vcal \subset \{0,1\}^K$. 
With this notation, we seek to compute
\begin{align*}
  v^\star \triangleq \argmax_{v \in \Vcal}\langle v, \mu\rangle.
\end{align*}
We are interested in learning algorithms that acquire information
about the unknown $\mu$ in an interactive, iterative fashion. At the
$t^{\textrm{th}}$ iteration, the learning algorithm selects an arm
$a_t$ and receives a corresponding observation $y_t \sim
\nu_{a_t}$. The algorithm's choice $a_t$ may depend on all previous
decisions and observations $\{(a_\tau,y_\tau)\}_{\tau=1}^{t-1}$ and
possibly additional randomness.

\ifthenelse{\equal{\version}{workshop}}{We study the
  \emph{fixed confidence setting}, where a failure probability
  parameter $\delta \in (0,1)$ is provided as input to the algorithm,
  which produces an estimate $\hat{v}$, but must further enjoy the
  guarantee that $\PP[\hat{v} \ne v^\star] \le \delta$. In this
  setting, we seek to minimize the number of queries issued by the
  algorithm.}{ We consider two related performance goals. In the
  \emph{fixed budget setting}, the learning algorithm is given a
  budget of $T$ queries, after which it must produce an estimate
  $\hat{v}$ of the true optimum $v^\star$, and we seek to minimize the
  probability of error $\PP[\hat{v} \ne v^\star]$.  In
  the \emph{fixed confidence setting}, a failure probability parameter
  $\delta$ is provided as input to the algorithm, which still produces
  an estimate $\hat{v}$, but must further enjoy the guarantee that
  $\PP[\hat{v} \ne v^\star] \le \delta$. In this setting, we seek to
  minimize the number of queries issued by the algorithm.}

Since we are performing optimization over $\Vcal$, for computational
efficiency, we equip our algorithms with a linear optimization oracle
for $\Vcal$. Formally, we assume access to a function
\begin{align}
\label{eq:oracle}
\oracle(c) \triangleq \argmax_{v \in \Vcal} \langle v, c \rangle,
\end{align}
that solves the offline combinatorial optimization problem. This
oracle is available in many combinatorial problems, including
max-weight matchings, spanning trees, and shortest/longest paths in
DAGs,\footnote{Other problems may require slight reformulations of the
  setup for computational efficiency. For example, shortest paths
  in undirected graphs requires $\mu_a \in [0,1]$ and considering a
  minimization version.} and is a basic requirement here since
otherwise even if $\mu$ were known we would not be able to find
$v^\star$. Technically, we allow the oracle to take one additional
constraint of the form $v_a = b$ for $a\in A, b\in \{0, 1\}$, which
preserves computational efficiency in most cases.

To fix ideas, we describe two concrete motivating examples 
\ifthenelse{\equal{\version}{workshop}}{}{(see also~\pref{sec:examples})}.
\begin{example}[\matching]
\label{ex:matching}
Consider a complete bipartite graph with $\sqrt{K}$ vertices in each
partition cell so that there are $K$ edges, which we identify with $A$. Let
$\Vcal$ denote the perfect bipartite matchings and let $\mu$ assign a
weight to each edge. Here, the combinatorial pure exploration task amounts
to finding the maximum-weight bipartite matching in a graph with edge
weights that are initially unknown. Note that the linear optimization
oracle~\pref{eq:oracle} is available here. 
\end{example}

\begin{example}[\biclique]
\label{ex:biclique}
In the same graph-theoretic setting, let $\Vcal$ denote the set of
bicliques with $\sqrt{s}$ vertices from each partition. Equivalently
in a $\sqrt{K}\times\sqrt{K}$ matrix, $\Vcal$ corresponds to all
submatrices of $\sqrt{s}$ rows and $\sqrt{s}$ columns. This problem is
variously referred to as biclique, biclustering, or submatrix
localization, and has applications in
genomics~\citep{wang2007biclustering}. Unfortunately,~\pref{eq:oracle}
is known to be NP-hard for this structure.
\end{example}

\paragraph{A new complexity measure.}
We define two complexity measures that govern the
performance of our algorithms. We start with the notion of a gap
between the decision sets:
\begin{align*}
\Delta_v(\mu) \triangleq \frac{\langle 
  v^\star - v,\mu\rangle}{d(v^\star,v)}.
\end{align*}
$\Delta_v(\mu)$ captures the difficulty of determining if $v$ is
better than $v^\star$, and the normalization $d(v^\star,v)$ accounts
for the fact that the numerator is a sum of precisely $d(v,v^\star)$
terms. The gap for arm $a$ is $\Delta_a(\mu) \triangleq \min_{v: a \in
  v^\star \ominus v} \Delta_v(\mu)$, which captures the difficulty of
determining if $a$ is in the optimal set. 

We also introduce complexity measures that are independent of
$\mu$. For $v \in \Vcal$ and $k \in \NN$, let $\Bcal(k,v) \triangleq
\{ u \in \Vcal \mid d(v,u) = k\}$ be the sphere of radius $k$ centered
at $v$. Then define
\begin{align*}
\Phi \triangleq \Phi(\Vcal) \triangleq \max_{k \in \NN, v \in \Vcal} \frac{\log(|\Bcal(k,v)|)}{k}.
\end{align*}
For some intuition, $\Phi$ measures the growth rate of $\Vcal$ as we expand away from some candidate $v$. 
Finally, let $\Psi \triangleq \Psi(\Vcal) \triangleq \min_{u,v \in \Vcal} d(u,v)$ denote
the smallest distance. In all of these definitions, we omit the
dependence on $\mu$ and $\Vcal$ when it is clear from context.

\paragraph{A deviation bound and the non-interactive setting.}
As a reference point and to foreshadow our results, we first 
study the non-interactive setting.
With budget $T$, a non-interactive algorithm queries each arm $T/K$
times and then outputs an estimate $\hat{v}$ of $v^\star$.\footnote{For
  simplicity, we do not implement a stopping rule, which na\"{i}vely
  incurs a $\max_{v \ne v^\star} \log(1/\Delta_v)$ dependence.}  In
this case, we have the following \emph{normalized
  regret inequality}, which will play a central role in our analysis. 
\begin{lemma}[Normalized Regret Inequality]
\label{lem:normalized_ineq}
Query each arm $T/K$ times and let $\hat{\mu} \in \RR^K$ be the vector
of sample averages. Then $\forall \delta \in (0,1)$,  
\begin{align}
\PP\rbr{\exists v \in \Vcal:\ \frac{\abr{\inner{v^\star - v}{\hat{\mu} - \mu}}}{d(v^\star,v)} \geq \sqrt{\frac{2K}{T} \cdot \rbr{\Phi + \frac{\log(2K/\delta)}{\Psi}}}} \leq \delta. \label{eq:normalized_ineq}
\end{align}
\end{lemma}

This inequality and a simple argument sharply characterize the
performance of the MLE, $\hat{v} = \argmax_{v \in \Vcal}
\inner{v}{\hat{\mu}}$, which we prove is nearly minimax optimal for
the non-interactive setting.  


\begin{theorem}[Non-interactive upper and lower bound]
  \label{thm:noninteractive_upper}
  For any $\mu \in [-1,1]^K$ and $\delta \in (0,1)$ the non-interactive MLE
  guarantees $\PP_{\mu}[\hat{v}\ne v^\star] \le \delta$ with
  \begin{align*}
    T = \order\left(\frac{K}{\min_{v\ne v^\star} \Delta_v^2}\left(\Phi + \frac{\log(K/\delta)}{\Psi}\right)\right).
  \end{align*}
  Further, with $\Scal(v^\star,\Delta) \triangleq \{\mu : v^\star = \argmax_{v \in \Vcal} \langle v, \mu\rangle, \forall v \ne v^\star, \Delta_v(\mu) \ge \Delta\}$, any non-interactive algorithm must have $\sup_{v^\star \in \Vcal} \sup_{\mu \in \Scal(v^\star,\Delta)}
  \PP_{\mu}[\hat{v} \ne v^\star] \ge 1/2$ as
  long as
  $T \le \frac{K}{ \Delta^2}\left(\Phi - \log \log 3\right)$.
\end{theorem}
See~\pref{app:noninteractive} for the proof
of~\pref{lem:normalized_ineq} and~\pref{thm:noninteractive_upper}. The
proofs are not difficult and involve adapting the argument
of~\citet{krishnamurthy2016minimax} to our setting. We also extend his
result by introducing the combinatorial parameters $\Phi,\Psi$, which
are analytically tractable in many cases.  Indeed, since $\Phi\leq
\log(K)$, a more interpretable, but strictly weaker, upper bound is
$\order\rbr{\frac{K\log(K/\delta)}{\min_v \Delta_v^2}}$.


To compare the upper and lower bounds,
note that $\Phi \geq 1/\Psi$ always but in most examples, including
\matching\xspace and \biclique, we actually have $\Phi \geq
\log(K)/\Psi$. As such, in the moderate confidence regime where
$\delta= \textrm{poly}(1/K)$, the upper and lower bounds disagree by
at most $\log(K)$ factor, but typically they agree up to
constants. Hence,~\pref{thm:noninteractive_upper} identifies the
minimax non-interactive sample complexity and we may conclude that the
MLE is near-optimal here.

\paragraph{Prior results.}
The departure point for our work is the observation that all prior
results for the interactive setting can be polynomially worse than the
bound for the MLE. We defer a detailed comparison
to~\pref{sec:examples}, but as a specific example, on
$\sqrt{s}$-\biclique,~\pref{thm:noninteractive_upper} can improve on
the bound of~\citet{chen2014combinatorial} by a factor of $s^{3/2}$,
and it can improve on the bounds of~\citet{gabillon2016improved}
and~\citet{chen2017nearly} by a factor of $\sqrt{s}$.



For intuition, the analyses of~\citet{chen2014combinatorial}
and~\citet{gabillon2016improved} involve a uniform convergence
argument over individual arms. As noted by~\citet{chen2017nearly},
this argument is suboptimal whenever $\Psi \gg 1$ as it does not take
advantage of large distances between hypotheses (note that $\Psi=
\sqrt{s}$ in \biclique).  The analysis of~\citet{chen2017nearly}
avoids this argument, but instead involves uniform convergence over
the decision set, which can be suboptimal when set differences to
$v^\star$ vary in size. (e.g., in \biclique, the minimum and maximum
distances are $O(\sqrt{s})$ and $\Omega(s)$ respectively).  In
comparison, our analysis \emph{always} gives the minimax optimal
non-interactive rate (up to a $\log(K/\delta)$ factor), reflecting the
advantage of the normalized regret inequality over these other proof
techniques. In the next section, we show how this inequality yields an
interactive algorithm that is never worse than the MLE but that can
also be substantially better.

\paragraph{Related work.}

\ifthenelse{\equal{\version}{workshop}}{}{
Combinatorial pure exploration generalizes the best arm identification
problem, which has been extensively studied
(c.f.,~\citep{even2006action,mannor2004sample,audibert2010best,karnin2013almost,russo2016simple,garivier2016optimal,carpentier2016tight,chen2016pure,simchowitz2017simulator}
for some classical and recent results). This problem is much simpler
both computationally and statistically than ours, and,
accordingly, the results are much more precise. One important difference is that in best
arm identification, verifying that the optimal solution is correct is
roughly as hard as finding the optimal solution, which motivates many
algorithms and lower bounds based on Le Cam's
method~\citep{kaufmann2014complexity,karnin2016verification,garivier2016optimal,chen2017nearly}. 
However, for combinatorial problems discovering the optimal solution
often dominates the sample complexity, and hence these techniques do
not immediately produce near-optimal results in the combinatorial
setting.
Nevertheless our algorithms are inspired by some ideas from this
literature, namely elimination and successive-reject
techniques~\citep{even2006action,audibert2010best}.

The subset selection problem, also called \topk, is a special case of
combinatorial exploration where $\Vcal$ corresponds to all ${K \choose
  s}$
subsets~\citep{kalyanakrishnan2012pac,bubeck2013multiple,kaufmann2013information}. This
case is minimally structured, and, in particular, there is little to
be gained from our approach since $\Phi = \Theta(\log(K))$. A
related effect occurs when the decision set corresponds to the basis
of a matroid~\citep{chen2016pure}.

Structure discovery has also been studied in related mathematical
disciplines including electrical engineering and statistics. Research
on \emph{adaptive sensing} from the signal processing community
studies a similar setup but with assumptions on the mean $\mu$, which
lead to more specialized algorithms that fail in our general
setup~\citep{tanczos2013adaptive}. Work from information theory and
statistics focuses on non-interactive versions of the problem and
typically considers specific combinatorial
structures~\citep{krishnamurthy2016minimax,balakrishnan2011noise,ariascastro2008searching}. In
particular, the \biclique\xspace problem is extensively studied in the
non-interactive setting and the minimax rate is
well-known~\citep{chen2014statistical,kolar2011minimax,butucea2013detection}.


Our fixed confidence algorithm is inspired by disagreement-based
active learning approaches, which eliminate inconsistent hypotheses
and query where the surviving ones
disagree~\citep{cohn1994improving,hsu2010algorithms,hanneke2014theory}.
Our algorithm is similar but uses a stronger elimination criteria,
leading to sharper results for exact identification.  Unfortunately,
exact identification is rather different from PAC-learning, and it
seems our approach yields no improvement for PAC-active learning.


Lastly, we use an optimization oracle as a computational
primitive. This abstraction has been used previously in combinatorial
pure exploration~\citep{chen2014combinatorial,chen2017nearly}, but
also in other information acquisition problems including active
learning~\citep{hsu2010algorithms,huang2015efficient} and contextual
bandits~\citep{agarwal2014taming,syrgkanis2016improved,rakhlin2016bistro}.}

\section{Results}

Pseudocode for our fixed confidence algorithm is given in~\pref{alg:main}. The algorithm proceeds in rounds, and at
each round it issues queries to a judiciously chosen subset of the
arms. These arms are chosen by implicitly maintaining a version space
of plausibly optimal hypotheses and checking for disagreement among
the version space. 

The key ingredient is the definition of the version space. For a
vector $\hat{\mu} \in \RR^K$ and a radius parameter $\Delta$, the
version space is defined as
\begin{align}
\label{eq:version_space}
\Vcal(\hat{\mu},\Delta)\triangleq \left\{v \in \textrm{conv}(\Vcal) \mid \forall u \in \Vcal, \langle \hat{\mu}, u-v\rangle \le \Delta \|u - v\|_1 \right\}.
\end{align}
Here $\textrm{conv}(\Vcal)$ is the convex hull of $\Vcal$ and
$\|u-v\|_1$ is the $\ell_1$ norm, which is just $d(u,v)$ for binary
$u,v$.  The version space is normalized in that the radius is
modulated by $\|u-v\|_1$, which is justified
by~\pref{eq:normalized_ineq}. This yields much sharper guarantees than the
more standard un-normalized definition $\{v \mid \max_{u \in \Vcal}
\langle \hat{\mu}, u - v\rangle \leq \Delta\}$ from the active
learning literature~\citep{cohn1994improving,hsu2010algorithms}.
At round $t$, the version space we use is $\Vcal_t \triangleq
\Vcal(\hat{\mu}_t,\Delta_t)$ where $\hat{\mu}_t$ is the empirical mean
vector and $\Delta_t$ is defined in the algorithm based on the right
hand side of~\pref{eq:normalized_ineq}.


This version space is used by the disagreement computation
(\pref{alg:disagreement}), which, with parameters $a \in A,
b \in \{0,1\}, \Delta, \hat{\mu}, \delta$ approximately solves the feasibility problem
\begin{align}
\label{eq:dis_feasibility}
? \exists v \in \Vcal(\hat{\mu},\Delta) \textrm{ s.t. } v(a) = b.
\end{align}
At round $t$, we use $\hat{\mu}_t,\Delta_t$, and the value for $b$ that
we use in~\pref{line:disagree} is $1 - \hat{v}_t(a)$, where
$\hat{v}_t$ is the empirically best hypothesis on $\hat{\mu}_t$. Since
$\hat{v}_t$, being the empirically best hypothesis, is always in
$\Vcal_t$, this computation amounts to checking if there exist
two surviving hypotheses that \emph{disagree} on arm $a$.
We use this disagreement-based criteria to drive the query strategy.

Before turning to computational considerations, a few other details
warrant some discussion. First, if at any round we detect that there
is no disagreement on some arm $a$, then we use a \emph{hallucinated}
observation $y_t(a) = 2\hat{v}_t(a)-1 \in \{\pm 1\}$. While this leads
to bias in our estimates, since all surviving hypotheses $v \in
\Vcal_t$ agree with $\hat{v}_t$ on arm $a$, this bias favors the
survivors. As in related work on disagreement-based active learning,
this helps enforces monotonicity of the version
space~\citep{dasgupta2007general}. Finally, we terminate once there are
no remaining arms with disagreement, at which point we output the
empirically best hypothesis.

\begin{figure*}
\noindent\ifthenelse{\equal{\version}{colt}}{\begin{minipage}{0.5\linewidth}}{\begin{minipage}{0.49\linewidth}}
\begin{algorithm}[H]
  \caption{Fixed Confidence Algorithm}
  \begin{algorithmic}[1]
    \State Input: Class $\Vcal$, failure probability $\delta \in (0,1)$
    \State Set $\Delta_t = \min\left\{1, \sqrt{\frac{8}{t}\left(\frac{\Phi\Psi + \log(\frac{K\pi^2t^2}{\delta})}{\Psi}\right)}\right\}$
    \State Sample each arm once $y_0(a) \sim \nu_a$
    \State Set $\hat{\mu}_1 = y_0$
   \For{$t=1,2\ldots,$}
    \State Compute $\hat{v}_t = \argmax_{v \in \Vcal} \langle v, \hat{\mu}_t\rangle$
    \For{$a \in [K]$}
    \State \textbf{if} \textsc{Dis}$(a,1-\hat{v}_t(a), \Delta_t, \hat\mu_t, \frac{\delta}{t^2\pi^2})$ \label{line:disagree}
    \State ~~~~Query $a$, set $y_t(a) \sim \nu_a$
    \State \textbf{else}
    \State ~~~~Set $y_t(a) = 2\hat{v}_t(a)-1$
    \State \textbf{endif}
    \EndFor
    \State Update $\hat{\mu}_{t+1} \gets \frac{1}{t+1}\sum_{i=0}^{t}y_i$
    \State If no queries issued this round, output $\hat{v}_t$
    \EndFor
    \vspace{0.1em}
    \end{algorithmic}
  \label{alg:main}
\end{algorithm}
\end{minipage}
\noindent\ifthenelse{\equal{\version}{colt}}{\begin{minipage}{0.5\linewidth}}{\begin{minipage}{0.49\linewidth}}
\begin{algorithm}[H]
\caption{Oracle-based Disagreement (\textsc{Dis})}
  \begin{algorithmic}[1]
    \vspace{0.065cm}
    \State Input: $a, b, \Delta, \hat{\mu}, \delta$
    \State Set $T = \frac{169K^3\log(4K/\delta)}{\Delta^2}$, $m = T\log(\frac{4KT}{\delta})$
    \State Set $\epsilon = \sqrt{\frac{1}{25KT\log(4K/\delta)}}$, $\ell_0=0$
    \For{$t=1, \dots,T$}
    \For{$i=1, \dots, m$}
    \State Sample $\sigma_{t, i} \sim \textrm{Unif}([0, 1/\epsilon]^K) $
    \State $u_{t, i}$ = \Call{Oracle}{$\Vcal, \sum_{t=0}^{t-1} \ell_t + \sigma_{t,i}$}
    \EndFor
    \State Let $s,x_t$ be the value and optimum of
    \begingroup\abovedisplayskip=0.2em \belowdisplayskip=0.4em
    \begin{align}
      \max {\textstyle\sum}_{i=1}^{m} \Delta \langle v, \one - 2 u_{t,i}\rangle + \langle v, \hat{\mu}\rangle\label{eq:empfeassmall}\\
      \textrm{s.t. } v \in \textrm{conv}(\Vcal), v(a) = b \notag
    \end{align}
    \endgroup
    \State $\textrm{val} = \sum_{i=1}^{m} \langle u_{t,i},\Delta\one - \hat{\mu}\rangle$
    \State if $s + \textrm{val} < 0$ \Return \false \label{line:infeas}
    \State Set $\ell_t = \Delta\one - 2\Delta x_t - \hat{\mu}$
    \EndFor
    \State \Return \true
    \vspace{0.05cm}
  \end{algorithmic}
  \label{alg:disagreement}
\end{algorithm}
\end{minipage}
\vspace{-0.4cm}
\end{figure*}

Note that to set $\Delta_t$ in the algorithm, we must compute
$\Phi$. This can be done analytically for many structures including,
paths in various graph models, bipartite matching, and the biclique
problem. Even when it cannot, $\Phi$ is independent of the unknown
means, so it can always be computed via enumeration, although this may
compromise the efficency of the algorithm. Finally, we can always use
the upper bound $\Phi \leq \log(K)$, which may increase the sample
complexity, but will not affect the correctness of the algorithm.

\paragraph{Efficient implementation of disagreement computation.}
Computationally, the bottleneck is the feasibility
problem~\pref{eq:dis_feasibility} for the disagreement
computation. All other computations in~\pref{alg:main} can trivially
be done in polynomial time with access to the optimization
oracle~\pref{eq:oracle}. Therefore, to derive an oracle-efficient
algorithm, we show how to solve~\pref{eq:dis_feasibility}, with
pseudocode in~\pref{alg:disagreement}.

It is not hard to see that~\pref{eq:dis_feasibility} is a linear
feasibility problem, but it has $|\Vcal|$ constraints, which could be
exponentially large. This precludes standard linear programming
approaches, and instead we use the Plotkin-Shmoys-Tardos reduction to
online learning~\citep{plotkin1995fast}.\footnote{Technically, we do
  have a separation oracle here, so we could use the Ellipsoid
  algorithm, but a standard application would certify feasibility or
  approximate infeasibility. Our reduction instead certifies
  infeasibility or approximate feasibility, which is more
  convenient.} The idea is to run an online learner to compute
distributions over the constraints and solve simpler feasibility
problems to generate the losses. In our case, the constraints are
parametrized by candidates $u \in \Vcal$ and we can express each
generated loss as a linear function of the constraint parameter $u$,
which enables us to use Follow-The-Perturbed-Leader (FTPL) as the
online learning algorithm~\citep{kalai2005efficient}. Importantly,
FTPL can be implemented using only the linear optimization oracle. As
a technical detail, we must use an empirical distribution based on
repeated oracle calls to approximate the true FTPL distribution, since
in our reduction the loss function is generated after and based on the
random decision of the learner.




First we provide the guarantee for the disagreement routine.

\begin{theorem}[Efficient Disagreement Computation]\label{thm:effdis}
\pref{alg:disagreement} with parameters
$a,b,\Delta,\hat{\mu}, \delta$ runs in polynomial time with
$\tilde{O}(K^6/\Delta^4)$ calls to \oracle. If it reports
\textsc{false} then Program~\pref{eq:dis_feasibility} is
infeasible. If it reports \textsc{true} then with probability at least
$1-\delta$, $\exists v\in \textrm{conv}(\Vcal), v(a)=b$ such that  $\forall u
\in \Vcal, \langle \hat{\mu} , u-v \rangle \le \Delta \|u - v\|_1
+\Delta$.
\end{theorem}

This result proves that~\pref{alg:disagreement} can approximate the
feasibility problem in~\pref{eq:dis_feasibility} in polynomial time
using the optimization oracle. The approximation is one-sided and,
since we do not query when the algorithm returns \false, only affects
the sample complexity of~\pref{alg:main}, but never the
correctness. The following theorem, which is the correctness and
sample complexity guarantee for~\pref{alg:main}, shows that this
approximation has negligible effect. For the theorem, recall the
definition of the arm gaps $\Delta_a \triangleq \min_{v:a \in v \ominus
  v^\star}\Delta_v$.


\begin{theorem}[Fixed confidence sample complexity bound]\label{thm:fixconsc}
For any combinatorial exploration instance with mean vector $\mu$, and
any $\delta \in (0,1)$,~\pref{alg:main} guarantees that $\PP[\hat{v}
  \ne v^\star] \le \delta$. Moreover, it runs in polynomial time with
access to the optimization oracle, and the total number of samples is
at most
\begin{align*}
\sum_{a \in K}\frac{144}{\Delta_a^2}\left(\Phi + \frac{2\log(144/(\Delta_a^2\Psi)) + 2\log(K\pi^2/\delta)}{\Psi}\right).
\end{align*}
\end{theorem}

The bound replaces the worst case gap, $\frac{K}{\min_v \Delta_v^2}$,
in~\pref{thm:noninteractive_upper} with a less pessimistic notion,
$\sum_a \Delta_a^{-2}$, that accounts for heterogeneity in the
problem.
Since $\sum_a \Delta_a^{-2} \leq \frac{K}{\min_v \Delta_v^2}$, the
bound is never worse than the minimax lower bound for non-interactive
algorithms (given in~\pref{thm:noninteractive_upper}) by more than a
logarithmic factor, but it can be much better if many arms have large
gaps. To our knowledge,~\pref{alg:main} is the first combinatorial
exploration algorithm that is never worse than non-interactive
approaches yet can exploit heterogeneity in the problem. 


\pref{thm:fixconsc} is not easily comparable with prior results for
combinatorial pure exploration, which use different complexity
measures than our gaps $\Delta_a$, $\Phi$, and $\Psi$. Our
observations from~\pref{thm:noninteractive_upper} apply here:
since~\pref{thm:fixconsc} precisely captures the combinatorial
structure of $\Vcal$, it can yield polynomial improvements over prior
work. On the other hand, our notion of gap $\Delta_a$ is different
from, and typically smaller than, prior definitions, so these results
can also dominate ours.
We defer a detailed comparison with calculations for several concrete examples to~\pref{sec:examples}.

We provide the full proof for~\pref{thm:effdis}
and~\pref{thm:fixconsc} in~\pref{app:proofs}. As a brief sketch,
we prove a martingale version of~\pref{lem:normalized_ineq}.
This inequality and the choice of
$\Delta_t$ verifies that $v^\star$ is never eliminated. We then show
that once $\Delta_t < \Delta_a$, all hypothesis $v \in \Vcal$ with $a
\in v^\star \ominus v$ satisfy $\inner{\hat{\mu}}{v^\star - v} >
\Delta_t \nbr{v^\star-v}_1$ and so they are eliminated from the
version space. \pref{thm:effdis} then guarantees that arm $a$ will
never be queried again, which yields the sample complexity bound.



\ifthenelse{\equal{\version}{workshop}}{
Briefly we mention three examples:
\begin{enumerate}
\item For classical best arm identification, all of these results are
  essentially the same, differing only in logarithmic factors.
\item For \topk\xspace and \matroid, all results are essentially the same in
  the homogeneous case where $\mu = (2v^\star-1)\Delta$ for some
  $v^\star \in \Vcal$. Outside of this special case, our bound can be
  worse since it uses a worse instance-dependent complexity measure. 
\item For problems like Matching and Biclique, as describe in~\pref{prop:comparison}, our bound can be polynomially
  better in the homogeneous case. In general, the results
  are incomparable since we have a worse instance-dependent measure
  but a better instance-independent one.
\end{enumerate}
}{} 

\ifthenelse{\equal{\version}{workshop}}{}{

\subsection{Deferred Results}
In this section we state two related results: a guarantee for a
disagreement-based algorithm in the fixed budget setting, and a more
refined sample complexity bound for a computationally inefficient
fixed confidence algorithm. Both algorithms and all proof details are
deferred to the appendices.

\subsubsection{A fixed budget algorithm}
Recall that in the fixed budget setting, the algorithm is given a
budget of $T$ queries and after issuing these queries, it must output
an estimate $\hat{v}$. The goal is to minimize $\PP[\hat{v} \ne
  v^\star]$. As is common in the literature, this setting requires a
modified definition of the instance complexity, which for our fixed
confidence result is $H \triangleq \sum_{a \in [K]}
\Delta_a^{-2}$. For the definition,
let $\Delta^{(j)}$ denote the $j^{\textrm{th}}$ largest $\Delta_a$
value, breaking ties arbitrarily.
The complexity measure for the fixed budget setting is
\begin{align*}
\tilde{H} \triangleq \max_j (K+1-j) (\Delta^{(j)})^{-2}.
\end{align*}
It is not hard to see that $\tilde{H} \le H \le \tlog(K) \tilde{H}$, where $\tlog(t) = \sum_{i=1}^t 1/i$ is the partial harmonic sum. With these new definitions, we can state our fixed budget
guarantee.
\begin{theorem}[Fixed Budget Guarantee]
\label{thm:fixed_budget}
Given budget $T\ge K$, there exists a fixed budget
algorithm (\pref{alg:fixbudget} in~\pref{app:fixed_budget}) that guarantees
\begin{align*}
\PP[\hat{v} \ne v] \le K^2 \exp \left\{\Psi \left(\Phi -  \frac{(T-K)}{9\tlog(K) \tilde{H}}\right)  \right\}.
\end{align*}
\end{theorem}
See~\pref{app:fixed_budget} for the proof. At a high level, the
savings over a na\"{i}ve analysis are similar to~\pref{thm:fixconsc}.
By using the normalized regret inequality, we obtain a refined
dependence on the hypothesis complexity, replacing $\log |\Vcal|$ with
the potentially much smaller $\log |\Bcal(k,v^\star)|$ (implicitly
through the $\Phi$ parameter). To compare, non-interactive methods
scale with $K(\Delta^{(1)})^{-2}$ instead of $\tilde{H}$, so the bound
is never worse, but it can yield an improvement when the arm gaps are
not all equal, which results in $\tilde{H} < K(\Delta^{(1)})^{-2}$.
Unfortunately, the
algorithm is not oracle-efficient.

\subsubsection{A refined fixed confidence guarantee}

We also derive a sharper sample complexity bound for the fixed
confidence setting. First, define
\begin{align*}
  D(v,v') \triangleq \max\{\log|\Bcal(d(v,v'),v)|,\log|\Bcal(d(v,v'),v')|\},
\end{align*}
to be the symmetric log-volume. We use two new instance-dependent complexity measures:
\begin{align*}
H^{(1)}_a \triangleq \max_{v: a \in v\ominus v^\star}\frac{d(v, v^\star)}{\langle \mu, v^\star - v \rangle^2}, \qquad H^{(2)}_a \triangleq \max_{v: a \in v \ominus v^\star} \frac{d(v,v^\star) D(v,v^\star)}{\langle \mu, v^\star - v\rangle^2}.
\end{align*}
The two definitions provide more refined control on the two terms
in~\pref{thm:fixconsc}. Specifically $H^{(1)}$ replaces the minimum
distance $\Psi$ with the distance to the hypothesis maximizing the
complexity measure. Similarly $H^{(2)}$ replaces the volume measure
$\Phi$ with a notion particular to the maximizing hypothesis.
Using these definitions, we have the following fixed-confidence
guarantee.
\begin{theorem}[Refined fixed confidence guarantee]\label{thm:refinefixcon}
There exists a fixed confidence algorithm (\pref{alg:inefffixcon} in~\pref{app:refined_proof}) that guarantees $\PP[\hat{v} \ne v^\star] \le \delta$ with sample complexity
\begin{align*}
T \le 64 \sum_{a\in[K]} H^{(1)}_{a} \left( 2\log(64H^{(1)}_{a}) + \log\frac{\pi^2 K}{\delta}\right) + 64 H^{(2)}_a.
\end{align*}
\end{theorem}
To understand this bound and compare with~\pref{thm:fixconsc}, note
that we always have $H_a^{(1)} \leq \rbr{\Delta_a^2\Psi}^{-1} $ and
$H_a^{(2)} \leq \frac{\Phi}{\Delta_a^2}$.  As
such,~\pref{thm:refinefixcon} improves on~\pref{thm:fixconsc} by
replacing worst case quantities $\Phi,\Psi$ with instance-specific
variants. However, the algorithm is not oracle-efficient.

\section{Examples and Comparisons}
\label{sec:examples}

As mentioned, the bound in~\pref{thm:fixconsc} is somewhat
incomparable to previous
results~\citep{chen2014combinatorial,gabillon2016improved,chen2017nearly}.
To provide general insights, we perform an
instance-independent, structure-specific analysis, fixing $\Vcal$ and
tracking combinatorial quantities but considering the least favorable
choice of $\mu$. Such an analysis reveals when one method dominates
another for hypothesis class $\Vcal$ \emph{for all mean vectors}, but
is less informative about specific instances. For a
complementary view, we study the \emph{homogeneous
  setting}, where $\mu = \Delta(2v^\star - \one)$ for some
$v^\star \in \Vcal$.

We also consider four specific examples.  The first is the
\topk\xspace problem, where $\Vcal$ corresponds to all ${K \choose s}$
subsets. The second is \textsc{DisjSet}, where there are $K$ arms and
$K/s$ hypotheses each corresponding to a disjoint set of $s$ arms,
generalizing an example of~\citet{chen2017nearly}.\footnote{As they
  argue, \textsc{DisjSet} is important because it can be embedded in other
  combinatorial structures, like disjoint paths.}
 The third and fourth
examples are \matching\xspace from~\pref{ex:matching} and
$\sqrt{s}$-\biclique\xspace from~\pref{ex:biclique}.

Throughout, we ignore constant and logarithmic factors, and we use
$\lesssim, \asymp$ to denote such asymptotic comparisons.  The main
combinatorial parameters are $\Psi \triangleq \min_{u,v \in \Vcal}
d(u,v)$, $D \triangleq \max_{u,v \in \Vcal} d(u,v)$, $\log |\Vcal|$,
and our combinatorial term $\Lambda \triangleq (\Phi + 1/\Psi)$.  $T$
denotes our sample complexity bound from~\pref{thm:fixconsc}, which we
instantiate in the the top row of~\pref{tab:homogeneous} for the four
examples in the homogeneous setting. All calculations are deferred
to~\pref{app:comparison}.


\paragraph{Comparison with~\citet{chen2014combinatorial}.}
The bound of~\citet{chen2014combinatorial} is $\tilde{O}\rbr{\sum_a
  \frac{\textrm{width}^2}{(\Delta^{(C)}_a)^2}}$, where $\Delta_a^{(C)}
\triangleq \min_{v: a \in v \ominus v^\star}\langle \mu, v^\star -
v\rangle$ is an unnormalized gap, and $\textrm{width}$ is the size of
the largest augmenting set in the best collection of augmenting sets
for $\Vcal$. In contrast, in our bound of
$\tilde{O}(\Lambda\sum_a\Delta_a^{-2})$, the normalized gaps that we
use incorporate the distance between sets and our combinatorial term
$\Lambda \lesssim 1$ is small. The structure-specific relationship
with our bound is
\begin{align*}
T_{\textrm{Chen14}}\cdot \frac{\Psi^2}{\textrm{width}^2}\Lambda \lesssim T \lesssim T_{\textrm{Chen14}}\cdot \frac{D^2}{\textrm{width}^2}\Lambda.
\end{align*}
Therefore, whenever $\textrm{width}$ is comparable with the diameter
$D$, our bound is never worse and, further, if $\Lambda < 1$ our bound
provides a strict improvement. \textsc{DisjSet} is precisely such an
example, where our bound is a factor of $s$ better than that
of~\citet{chen2014combinatorial} for all instances. On the other hand,
in \topk, we have $\textrm{width} = \Psi = 1$ but $D = 2s$ and
$\Lambda \asymp 1$ so our bound is never better, but can be worse by a
factor of up to $s^2$, depending on $\mu$. As a final observation, in
the homogeneous case, we have $\Delta_a \geq
\frac{\Delta_a^{(C)}}{\textrm{width}}$ and so our bound is never worse
and is strictly better whenever $\Lambda < 1$.

Turning to the examples in the homogeneous case, we instantiate the
bound of~\citet{chen2014combinatorial} in the second row
of~\pref{tab:homogeneous}.  Our bound matches theirs for \topk\xspace and is
polynomially better for \textsc{DisjSet}, \matching, and \biclique.

\paragraph{Comparison to~\citet{gabillon2016improved}.}
\citet{gabillon2016improved} use a normalized definition of hypothesis
complexity similar to our $\Delta_v(\mu)$, but they compare each
hypothesis $v$ to its complement $C_v \in \Vcal$ where $C_v \triangleq
\argmax_{v' \ne v} \frac{\langle \mu, v' - v\rangle}{d(v',v)}$. They
then define an arm complexity as $\Delta^{(G)}_a \triangleq \min_{v: a \in v
  \ominus C_v}\frac{\langle \mu, C_v - v\rangle}{d(C_v,v)}$, and
obtain the final bound $\tilde{O}\rbr{\sum_a(\Delta_a^{(G)})^{-2}}$.  In
contrast, we always compare $v$ with $v^\star$ and so $\Delta_a \leq
\Delta_a^{(G)}$, but our bound exploits favorable structural
properties of the hypothesis class by scaling with $\Lambda$, which is
small. The structure-specific relationship is
\begin{align*}
T_{\textrm{Gabillon16}}\Lambda \lesssim T \lesssim T_{\textrm{Gabillon16}}\cdot \frac{D^2}{\Psi^2}\Lambda.
\end{align*}
Three observations from above apply here as well: (1) For
\textsc{DisjSet}, our bound yields a factor of $s$ improvement on all
instances, (2) for \topk, our bound is never better but can be a
factor of $s^2$ worse on some instances, and (3) our bound is never
worse in the homogeneous case and is an improvement whenever $\Lambda
< 1$. On the specific examples in the homogeneous case, we obtain a
polynomial improvement on \biclique, where $\Psi=1/\sqrt{s}$
(See~\pref{tab:homogeneous}).


\begin{table}
\label{tab:homogeneous}%
\begin{center}%
\renewcommand{\arraystretch}{1.2}
  \begin{tabular}{r | c | c | c | c}%
    Sample complexity & \topk & \textsc{DisjSet} & \matching & \biclique \\
    \hline\hline
    \pref{thm:fixconsc} & $\Theta(K)$ & $\Theta(K/s)$ & $\order(K)$ & $\order(K/\sqrt{s})$ \\
    \hline
    \citet{chen2014combinatorial} & $\Theta(K)$ & $\Theta(K)$ & $\Theta(K^2)$ & $\Theta(Ks)$ \\
    \hline
    \citet{chen2017nearly} & $\Theta(K)$ & $\Theta(K/s)$ & $\Omega(K^{3/2})$ & $\Omega(\sqrt{Ks} + K/\sqrt{s})$ \\ 
    \hline
    \citet{gabillon2016improved} & $\Theta(K)$ & $\Theta(K)$ & $\Theta(K)$ & $\Theta(K)$
  \end{tabular}%
\end{center}%
\ifthenelse{\equal{\version}{arxiv}}{}{\vspace{-0.5cm}}
\caption{Guarantees for four algorithms on specific
  examples, with $\mu =
  \Delta (2v^\star - 1)$, ignoring logarithmic factors. All bounds
  scale with $1/\Delta^2$, which is suppressed. For homogeneous
  problems,~\pref{thm:fixconsc} is never worse than prior results and
  can be polynomially better.}
\vspace{-0.5cm}
\end{table}

\paragraph{Comparison to~\citet{chen2017nearly}.}
Finally,~\citet{chen2017nearly} introduce a third arm complexity
parameter based on the solution to an optimization problem, which they
call $\Low$. They prove a fixed-confidence lower bound of $\Omega(\Low
\log(1/\delta))$ and an upper bound of $\tilde{O}\rbr{\Low \log
  (|\Vcal|/\delta)}$.
In general, the sharpest
structure-specific relationship we can obtain is
\begin{align*}
T_{\textrm{Chen17}}\cdot\frac{\Lambda}{\log |\Vcal|} \lesssim T \lesssim T_{\textrm{Chen17}}\cdot \frac{D^2}{\log |\Vcal|}\Lambda.
\end{align*}
The second inequality results from a rather crude lower bound on
$\Low$ and is therefore quite pessimistic.  Indeed, the factor on the
right hand side is always at least $1$, so this bound does not reveal
any structure where we can improve on~\citet{chen2017nearly} for all
choices of $\mu$.  Unfortunately it is difficult to
relate $\Low$ to natural problem parameters in general.

On the other hand, in the homogeneous case, we can bound $\Low$
precisely in examples, yielding the third row
of~\pref{tab:homogeneous}.
Roughly speaking, our bound replaces $\log |\Vcal|/\Psi$ with
$\Lambda$, which is always smaller, leading to polynomial improvements
in \biclique\xspace, when~$s > \sqrt{K}$, and \matching.

\paragraph{Final Remarks.}
We close this section with some final remarks. 
First, our result shows that $\log(|\Vcal|)$ dependence is not
necessary for many structured classes. This does not contradict the
lower bound in Theorem 1.9 of~\citet{chen2017nearly}, which constructs
certain pathological classes.

Second, we believe that the optimization-based measure \textrm{Low},
corresponds to the sample complexity for verifying that a proposed $v$
is optimal. Indeed the bound of~\citet{chen2017nearly} is optimal in
the extremely high-confidence setting ($\delta \leq 1/|\Vcal|$), where
high-probability verification dominates the sample complexity, yet it
is more natural to consider polynomially- rather than
exponentially-small $\delta$. In the moderate-confidence case,
exploration to find a suitable hypothesis $v$ is the dominant cost,
but the upper bound of~\citet{chen2017nearly} can be suboptimal
here. We believe the $\Omega(\textrm{Low}\log(1/\delta))$ lower bound
is loose in this regime, but are not aware of sharper lower bounds.

Finally, we note that for the fixed-confidence setting it is easy to
achieve the best of all of these guarantees, simply by running the
algorithms in parallel. For example, by interleaving queries issued
by~\pref{alg:main} and the algorithm of~\citet{chen2017nearly}, we
obtain a sample complexity of $2\min\{T, T_{\textrm{chen17}}\}$ with
probability $1-2\delta$. This yields an algorithm that is never worse
than the non-interactive minimax optimal rate \emph{and} is
instance-optimal in the high-confidence regime.



}

\section{Discussion and Open Problems}
This paper derives new algorithms for combinatorial pure
exploration. The algorithms represent a new sample complexity
trade-off and importantly are never worse than any non-interactive
algorithm, contrasting with prior results. Moreover, our fixed
confidence algorithm can be efficiently implemented whenever the
combinatorial family supports efficient linear optimization.

We close with some open problems. 
\begin{itemize}
\item In the homogeneous \biclique, our bound is
  $O(K/(\sqrt{s}\Delta^2))$, yet one can actually achieve
  $O(\frac{1}{\Delta^2}(\sqrt{K} +
  K/s))$ with a specialized algorithm~\citep{tanczos2013adaptive}. 
 Whether the faster rate is achievable
  beyond the homogeneous case is open, and seems related to the
  fact that active learning at best provides distribution-dependent
  savings in general but can provide exponential savings with random
  classification noise (analogous to our homogeneous setting).



\item Relatedly, settling the optimal sample complexity for
  combinatorial pure exploration is open. For lower bounds, the
  technical barrier is to capture the multiple testing phenomena,
  which typically requires Fano's Lemma. 
For upper bounds, some
  interesting algorithms to study are
  median-elimination~\citep{even2006action},
  explore-then-verify~\citep{karnin2016verification}, and
  sample-and-prune~\citep{chen2016pure}, all of which yield
  optimal algorithms in special cases.
\end{itemize}
We hope to study these questions in future work.



\section*{Acknowledgements}
We thank Sivaraman Balakrishnan for formative and insightful discussion. 
AK is supported in part by NSF Award IIS-1763618.
\newpage

\appendix
\section{Non-interactive analysis}
\label{app:noninteractive}

\subsection{Proof of~\pref{lem:normalized_ineq}}
Observe that $\frac{\inner{v^\star - v}{\hat{\mu} -
    \mu}}{d(v^\star,v)}$ is the average of $\frac{T}{K}d(v^\star,v)$
centered sub-Gaussian random variables, each with variance parameter
$1$. This follows because $v^\star - v \in \{-1,0,+1\}$ is non-zero on
exactly $d(v^\star,v)$ coordinates, and because $\hat{\mu} - \mu$ is
the average of $T/K$ sub-Gaussian random vectors. Therefore, by
a Subgaussian tail bound  and a union bound
\begin{align*}
\PP\sbr{\exists v \in \Vcal : \frac{\abr{\inner{v^\star - v}{\hat{\mu} - \mu}}}{d(v^\star,v)} \geq \epsilon} &\leq \sum_{v \in \Vcal} \PP\sbr{ \frac{\abr{\inner{v^\star - v}{\hat{\mu} - \mu}}}{d(v^\star,v)} \geq \epsilon }\\
& \leq 2 \sum_{v \in \Vcal} \exp\rbr{ \frac{-Td(v^\star,v)\epsilon^2}{2 K} }\\
& = 2\sum_{k=\Psi}^K \abr{\Bcal(k,v^\star)} \exp\rbr{\frac{-Tk\epsilon^2}{2K}} \\
& \leq 2K \exp\rbr{\max_{\Psi \leq k \leq K} \log \abr{\Bcal(k,v^\star)} - \frac{Tk\epsilon^2}{2K}}.
\end{align*}
Unpacking the definitions of $\Psi,\Phi$ and setting $\epsilon =
\sqrt{\frac{2K}{T}\rbr{\Phi + \frac{\log(2K/\delta)}{\Psi}}}$, this
bound is at most $\delta$, which proves the lemma.

\subsection{Proof of upper bound in~\pref{thm:noninteractive_upper}}
Recall the definition of the gap $\Delta_v(\mu)$, and observe that
probability of error for the MLE $\hat{v} = \argmax_{v \in \Vcal}\inner{\hat{\mu}}{v}$ is
\begin{align*}
\PP\sbr{\hat{v} \ne v^\star} &\leq \PP \sbr{ \exists v \in \Vcal: \inner{\hat{\mu}}{ v - v^\star} > 0} = \PP \sbr{ \exists v \in \Vcal: \inner{\hat{\mu} - \mu}{v - v^\star} > \inner{\mu}{v^\star - v}}\\
& = \PP \sbr{ \exists v \in \Vcal: \frac{\inner{\hat{\mu} - \mu}{v - v^\star}}{d(v,v^\star)} > \Delta_v(\mu) }\\
& \leq \PP \sbr{ \exists v \in \Vcal: \frac{\inner{\hat{\mu} - \mu}{v - v^\star}}{d(v,v^\star)} > \min_{v \ne v^\star} \Delta_v(\mu) }.
\end{align*}
Now the result follows from~\pref{lem:normalized_ineq}, specifically
by setting the right hand side of the normalized regret inequality to
be at most $\min_{v \ne v^\star} \Delta_v(\mu)$ and solving for $T$.

\subsection{Proof of lower bound in~\pref{thm:noninteractive_upper}}
The proof here is based of Fano's inequality and follows the analysis
of~\cite{krishnamurthy2016minimax}. Let us simplify notation and
define $P_v = \PP_{\mu = \Delta (2v-\one)}$ to be the distribution
where $T/K$ samples are drawn from each arm and $\mu= \Delta(2v -
\one)$. For any distribution $\pi$ supported on $\Vcal$ let $P_\pi$
denote the mixture distribution where first $v^\star \sim \pi$ and
then the samples are drawn from $P_{v^\star}$. With this notation,
Fano's inequality (with non-uniform prior) shows that for any
algorithm
\begin{align*}
  \sup_{v^\star \in \Vcal} \PP_{v^\star}[\hat{v} \ne v^\star] \ge \EE_{v^\star \sim \pi} \PP_{\mu = \Delta (2v^\star-\one)}[\hat{v} \ne v^\star] \ge 1 - \frac{\EE_{v \sim \pi} KL(P_v || P_{\pi}) + \log 2}{H(\pi)}.
\end{align*}
(This slightly generalizes one standard version of Fano's inequality,
where $\pi$ is uniform over $\Vcal$, so the denominator is $\log
|\Vcal|$.)  Let $\tilde{v} \in \Vcal$ denote the candidate achieving
the maximum in the definition of $\Phi$ and define the prior
\begin{align*}
\pi(v) \propto \exp\left( -\frac{T\Delta^2\|\tilde{v} - v\|_2^2}{K}\right).
\end{align*}
With this definition, the entropy term becomes
\begin{align*}
H(\pi) &= \log \left(\sum_{v} \exp\left( -\frac{T\Delta^2 \|\tilde{v} - v\|_2^2}{K}\right)\right) + \sum_v\pi(v) \frac{T\Delta^2\|v - \tilde{v}\|_2^2}{K}\\
& = \log \left(\sum_{v} \exp\left( -\frac{T\Delta^2 \|\tilde{v} - v\|_2^2}{K}\right)\right) + 2 \sum_v\pi(v)KL(P_v || P_{\tilde{v}}).
\end{align*}
Here in the last step we use the definition of the Gaussian KL, and
the tensorization property for KL-divergence. As for the
KL term in the numerator, it is not too hard to see that
\begin{align*}
\sum_v \pi(v) KL(P_v || P_\pi) \le \sum_v \pi_v KL(P_v || P_\pi) + KL(P_\pi || P_{\tilde{v}})  = \sum_v \pi_v KL(P_v || P_{\tilde{v}}).
\end{align*}
Thus, we have proved the lower bound if
\begin{align*}
  \sum_v \pi(v) KL(P_v||P_{\tilde{v}}) + \log(2) \le \frac{1}{2}\log \left(\sum_{v} \exp\left( -\frac{T\Delta^2 \|\tilde{v} - v\|_2^2}{K}\right)\right) + \sum_v\pi(v)KL(P_v || P_{\tilde{v}}).
\end{align*}
After simple algebraic manipulations, we get
\begin{align*}
\log(2 \log(2)) \le \sum_v \exp\left( - \frac{T\Delta^2}{K} d(v,\tilde{v})\right) = \sum_k \exp\left(\log |\Bcal(k,\tilde{v})| - \frac{Tk\Delta^2}{K}\right).
\end{align*}
Since the sum dominates the maximum, $\tilde{v}$ realizes
the definition of $\Phi$, and since $2\log(2) \le 3$, we obtain the
result. More formally, if
\begin{align*}
T \le \frac{K}{\Delta^2}\left(\Phi - \log\log 3\right),
\end{align*}
then
\begin{align*}
\sum_k \exp\left(\log |\Bcal(k,\tilde{v})| - \frac{Tk\Delta^2}{K}\right) &\ge \max_k \exp\left(\log |\Bcal(k,\tilde{v})| - \Phi k + \log \log 3 \right)\\
& = \exp(\log \log 3) = \log 3.
\end{align*}
Thus if $T$ is smaller than above, the minimax probability of error is
at least $1/2$.

\begin{remark}
As we have discussed, the lower bound identifies the minimax rate up
to constants for examples where $\Phi \geq \log(K)/\Psi$, in the
moderate confidence regime where $\delta =
\textrm{poly}(1/K)$. Obtaining the optimal $\delta$ dependence even
for non-interactive algorithms, seems quite challenging and is an
intriguing technical question for future work.
\end{remark}

\begin{remark}
We emphasize here that the lower bound applies only for
non-interactive algorithms and only in the homogeneous case. A more
refined instance-dependent bound is possible with our technique but is
not particularly illuminating.
\end{remark}

\section{Calculations for examples}
\label{app:comparison}

\subsection{Instantiations of~\pref{thm:fixconsc}}
To instantiate~\pref{thm:fixconsc} for the examples, we need to
compute $\Phi, \Psi$, and $\Delta_a$ for each arms $a$. In the
homogeneous case, we always have $\Delta_a = \Delta$. We now compute $\Phi,\Psi$ for the four examples. 

For \topk, we have $\Psi=2$ and
\begin{align*}
\Phi = \max_{k} \frac{\log\rbr{{K-s \choose k} {s\choose s-k}}}{k} \leq O\rbr{\log(K)}.
\end{align*}
Thus the sample complexity is $T \lesssim \frac{K}{\Delta^2}$, where recall that $\lesssim$ ignores logarithmic factors.

For \textsc{DisjSet}, it is easy to see that $\Phi \asymp 1/(2s)$,
$\Psi = 2s$, so we have $T \lesssim
\frac{K}{s\Delta^2}$.

For \matching, $\Psi =4$ since we must switch at least two edges to
produce another perfect matching. To calculate $\Phi$, by symmetry we
may assume that the ``center" $v$ is the identity matching
$\{(a_1,b_1)\}_{i=1}^{\sqrt{K}}$ where $\{a_i\}, \{b_i\}$ form the two
partition cells. Then, as an upper bound, the number of matchings that
differ on $4s$ edges is at most ${\sqrt{K} \choose s} s! \leq
K^s$. (Actually this bounds the ball volume and hence the sphere
volume.) Thus $\Phi \leq O(\log(K))$ and so we have $T \lesssim
\frac{K}{\Delta^2}$.



For \biclique, $\Psi = 2\sqrt{s}$ which arises by swapping a single
node on either side of the partition. The computation of $\Phi$ is
more involved.  The idea is that for every vertex that we swap into the
biclique, we switch $\Theta(\sqrt{s})$ edges, formally at least
$\sqrt{s}/2$ edges but no more than $\sqrt{s}$. Then rather than
optimizing over the radius in the decision set, we optimize over the
number of vertices swapped in on both sides of the partition, which we
denote $s_L,s_R$.  For a set $v$, note that we can obtain any set by
swapping $s_L$ column and $s_R$ rows of $v$
\begin{align*}
\max_{k} \frac{1}{k} \left(\log|\Bcal(k,v)|\right) &\le \max_{s_R,s_L} \frac{2}{\sqrt{s} (s_L+s_R)}
\log\left( {\sqrt{K}\choose s_L} {\sqrt{K} \choose s_R}\right) \leq O\rbr{\frac{\log(K)}{\sqrt{s}}}.
\end{align*}
This gives $T \lesssim \frac{K}{\sqrt{s}\Delta^2}$.

\subsection{Comparison with~\cite{chen2014combinatorial}.}
To compare with~\cite{chen2014combinatorial}, we must introduce some
of their definitions. Translating to our terminology, they define
$\Delta^{(C)}_a = \min_{v: a \in v \ominus v^\star} \langle\mu, v^\star - v\rangle$,
which differs from our definition since it is not normalized.  They
also define exchange classes and a notion of width of the decision
set. An exchange class is a collection of patches $b = (b_+,b_-)$
where $b_+,b_- \subset [K]$ and $b_+\cup b_- = \emptyset$, with
several additional properties. To describe them further define the
operator $v \oplus b = (v \setminus b-) \cup b_+$ and $v \oslash b =
(v\setminus b_+) \cup b_-$ where $v$ is interpreted as a subset of
$[K]$. Then a set of patches $\Bcal$ is an exchange class for $\Vcal$
if for every pair $v \ne v' \in \Vcal$ and every $a \in v\setminus
v'$, there exists a patch $b \in \Bcal$ such that (1) $a \in b_-$, (2)
$b_+ \subset v'\setminus v$, (3) $b_-\subset v \setminus v'$, (4)
$v\oplus b \in \Vcal$, and (5) $v'\oslash b \in \Vcal$. Then they
define the width
\begin{align*}
\textrm{width}(\Vcal) = \min_{\textrm{exchange classes } \Bcal} \max_{b
  \in \Bcal} |b_-| + |b_+|
\end{align*}
With these definitions, the fixed-confidence bound
of~\cite{chen2014combinatorial} is
\begin{align*}
\otil\left(\textrm{width}(\Vcal)^2\sum_a \frac{1}{(\Delta_a^{(C)})^2}\log(K/\delta)\right)
\end{align*}
where we have omitted a logarithmic dependence on the arm complexity
parameter $\Delta^{(C)}_a$.

For homogeneous \textsc{DisjSet} it is easy to see that
$\textrm{width}(\Vcal) \asymp s$ and $\Delta_a^{(C)} \asymp
s\Delta$. Hence their bound is $O(K\log(K)/\Delta^2)$.

For \matching,
number
the vertices on one side $a_1,\ldots,a_{\sqrt{K}}$ and on the other
side $b_1,\ldots,b_{\sqrt{K}}$. Let $v^\star$ be the matching with
edges $\{(a_i,b_i)\}_{i=1}^{\sqrt{K}}$. In the homogeneous case where
$\mu = \Delta (2v^\star-1)$, it is easy to see that $\Delta_a^{(C)} =
\Theta(\Delta)$ since for every edge $e$ (which correspond to the arms
in the bandit problem), there exist a matching that contains this
edge, that disagrees with $v^\star$ on exactly two
edges. Specifically, if $e = (a_i,b_j)$ then the matching that has
edge $(a_k,b_k)$ for all $k \ne i,j$ and edges $(a_i,b_j)$ and
$(a_j,b_i)$ has symmetric set difference exactly $4$.

On the other hand we argue that the width is $\Theta(\sqrt{K})$. This
is by the standard augmenting path property of the matching
polytope. In particular if $v^\star$ is as above and we define another
matching $v = \{(a_i,b_{i+1\mod \sqrt{K}})\}_{i=1}^{\sqrt{K}}$, then
the only patch for $v^\star,v$ is to swap all edges. Hence the bound
of~\cite{chen2014combinatorial}, in this instance is
$\otil\left( \frac{K^2}{\Delta^2}\log(K/\delta)\right)$
which is a factor of $K$ worse than the non-interactive algorithm in
this setting.

For \biclique\xspace, we have $\textrm{width}(\Vcal) \asymp s$ yet
$\Delta_a^{(C)} \asymp \sqrt{s}\Delta$. For the former, consider two
  bicliques $v,v'$ that disagree on all nodes on both sides of the
  partition. Then the smallest patch betweeen them is the trivial one
  $(v,v')$ since any other potential patch covers edges between the
  two bicliques (which are not contained in either one). For the
  latter, for any edge $a$ we can swap at most two nodes, one from
  each side, to cover this edge. Thus their bound is $O(Ks/\Delta^2)$.

For the worst case comparison, note that 
\begin{align*}
\Delta_a &= \min_{v: a \in v \ominus v^\star} \frac{\inner{\mu}{v^\star-v}}{d(v^\star,v)} \geq \min_{v: a \in v \ominus v^\star} \frac{\inner{\mu}{v^\star-v}}{D} = \frac{\Delta_a^{(C)}}{D}\\
\Delta_a & \leq \frac{\inner{\mu}{v^\star - v_a}}{d(v^\star,v_a)} \leq \frac{\Delta_a^{(C)}}{\Psi},
\end{align*}
where $v_a$ is the set that witnesses $\Delta_a^{(C)}$ and $D
\triangleq \max_{u,v\in \Vcal} d(u,v)$ is the diameter. Ignoring
logarithmic factors, our bound therefore satisfies
\begin{align*}
T_{\textrm{chen14}}\frac{\Psi^2}{\textrm{width}^2}\Lambda \leq T \leq T_{\textrm{chen14}}\frac{D^2}{\textrm{width}^2}\Lambda.
\end{align*}

\subsection{Comparison with~\cite{chen2017nearly}.}
As for~\cite{chen2017nearly}, their guarantee is
\begin{align*}
\otil\left(\Low(\Vcal) (\log(1/\delta) + \log |\Vcal| )\right),
\end{align*}
ignoring some logarithmic factors. Here $\Low(\Vcal)$ is the solution
to the optimization problem
\begin{align}
  \textrm{minimize} &\sum_a \tau_a
  \textrm{ s.t.} \sum_{a \in v \ominus v^\star} \frac{1}{\tau_a} \le \langle \mu, v^\star - v\rangle^2, \forall v \ne v^\star
  \textrm{ and } \tau_a \ge 0, \forall a \in [K].
\label{eq:low_program}
\end{align}
In the homogeneous case for bipartite matching, we show that
$\Low(\Vcal) = \Theta(K/\Delta^2)$. This proves what we want since
$\log(\Vcal) \asymp \sqrt{K}$ and hence the bound is a factor of
$\sqrt{K}$ worse than~\pref{thm:noninteractive_upper}.

The proof here is by passing to the dual of
Program~\pref{eq:low_program}. First we construct the Lagrangian
\begin{align*}
\Lcal (\tau, \alpha) =  \sum_a \tau_a + \sum_v \alpha_v \left( \sum_{a \in v \ominus v^\star}  \frac{1}{\tau_a} - \langle \mu, v^\star - v\rangle^2 \right).
\end{align*}
By weak duality, the solution of the primal problem is always lower
bounded by the solution of the dual problem
\begin{align*}
\min_\tau \max_\alpha \Lcal(\tau, \alpha) \ge \max_\alpha \min_\tau  \Lcal (\tau, \alpha).
\end{align*}
Taking the derivative with respect to $\tau$ we have
\begin{align*}
\frac{\partial\Lcal}{\partial \tau_a} = 1 - \sum_{v: a \in v \ominus v^\star} \alpha_v \left( \frac{1}{\tau_a^2} \right) = 0 \Rightarrow
\tau_a = \sqrt{\sum_{v: a \in v \ominus v^\star}\alpha_v},
\end{align*}
and plugging back into the Lagrangian gives
\begin{align}
\max_{\alpha_v \succeq 0} &\ \sum_a \sqrt{\sum_{v: a \in v \ominus v^\star}\alpha_v} + \sum_v \alpha_v\left(\sum_{a \in v \ominus v^\star}  \frac{1}{\sqrt{\sum_{v': a \in v' \ominus v^\star}
\alpha_{v'}}} -  \langle \mu, v^\star - v\rangle^2 \right)\notag \\
= \max_{\alpha_v \succeq 0} &\ 2 \sum_a \sqrt{\sum_{v: a \in v\ominus v^\star} \alpha_v} - \sum_v \alpha_v \langle \mu, v^\star -v\rangle^2. \label{eq:chen_dual}
\end{align}
By weak duality, any feasible solution here provides a lower bound on
$\Low(\Vcal)$. For matchings, we construct a feasible solution in a
similar way to the construction we used to analyze the bound
of~\cite{chen2014combinatorial}. Let $v^\star$ be the matching
$\{(a_i,b_i)\}_{i=1}^{\sqrt{K}}$. For every edge $(a_i,b_j)$, there is
a unique matching $v$ that disagrees with $v^\star$ on exactly $4$
edges, and for these matchings we will set $\alpha_v$ to some constant
value $\alpha$. We set $\alpha_v=0$ otherwise. This ensures that for
every $a \notin v^\star$, $\sum_{v: a \in v\ominus v^\star} \alpha_v =
\alpha$. On the other hand, for $a \in v^\star$, we get $\sum_{v: a
  \in v\ominus v^\star} \alpha_v = (\sqrt{K}-1)\alpha \ge \alpha$,
since we can swap out this edge with one of $\sqrt{K}-1$ other edges,
iterating over all other nodes on the other side of the partition. In
other words, for every arm $a \in [K]$, the first term is at least
$\sqrt{\alpha}$, while no more than $K$ $\alpha_v$s are non-zero. In
total, a lower bound on the dual program is given by
\begin{align*}
  \max_{\alpha \ge 0} 2K\sqrt{\alpha} - 4K\Delta^2\alpha.
\end{align*}
This simpler program is optimized with $\alpha = 1/(16\Delta^4)$ and
plugging back in reveals that
\begin{align*}
\Low(\Vcal) = \Omega(K/\Delta^2).
\end{align*}
This is all we need for our comparison, since $\Low(\Vcal)
\log(|\Vcal|) = \Omega(K^{3/2}/\Delta^2)$ in this case.


For \biclique, let us assume that $\sqrt{s}$ divides
$\sqrt{K}$. Considering the dual program \pref{eq:chen_dual}, we set 
$\alpha_v$ in the following way: We define the set that contains all the   
hypotheses that can be obtained by swapping the first row or the first column of $v^\star$ with another row or column to be set $V_c$. Note that 
$|V_c| = 2(\sqrt{K} - \sqrt{s})$. For $v\in V_c$, we set $\alpha_v = \alpha_1$.
We define a maximum set of disjoint hypotheses that does not share any rows or columns with
$v^\star$ to be $V_q$. Note that $|V_q| = (\frac{\sqrt{K} - \sqrt{s}}{\sqrt{s}})^2$. For $v\in V_q$ we set $\alpha_v = \alpha_2$. 
We discard all the other sets and set
$\alpha_v = 0$. For the remaining sets, note that for each arm $a\notin v^\star$, there is only one hypothesis $v$ such that that $a \in v^\star \ominus v$, let us call it $v_a$.
Thus the dual program (\ref{eq:chen_dual}) can be lower bounded as follow
\begin{align}
\text{(\ref{eq:chen_dual}) } \geq \max_{\alpha_1, \alpha_2} \sum_{\{a| v_a \in V_c\} } \sqrt{\alpha_1} + \sum_{\{a| v_a \in V_q\} } \sqrt{\alpha_2} - \sum_{v\in V_c} \alpha_1 s\Delta^2 - \sum_{v\in V_q} \alpha_2 s^2\Delta^2
\end{align}
Note that $|\{a \mid v_a \in V_c\}| = 2(\sqrt{K}-\sqrt{s})\sqrt{s}$ and $|\{a \mid v_a \in V_q\}| = (\sqrt{K}-\sqrt{s})^2$. 
Solving for $\alpha_1, \alpha_2$ gives $\alpha_1 = 1/s\Delta^4$, $\alpha_2 = 1/s^2\Delta^4$, and plugging these back into the dual gives
\begin{align*}
\Low(\Vcal) = \Omega\rbr{\frac{1}{\Delta^2}\rbr{\sqrt{K}-\sqrt{s} + \frac{(\sqrt{K}-\sqrt{s})^2}{s}}}.
\end{align*}
Recall that their sample complexity is $\Low(\Vcal)\log(\Vcal)$, where
$\log(\Vcal) \asymp \sqrt{s}$. This means that their sample complexity
is lower bounded as
\begin{align*}
\frac{(\sqrt{K} - \sqrt{s}) \sqrt{s}}{\Delta^2} + \frac{(\sqrt{K} - \sqrt{s})^2}{\sqrt{s}\Delta^2} = \Omega\rbr{\frac{1}{\Delta^2}\rbr{\sqrt{Ks} + \frac{K}{\sqrt{s}}}}
\end{align*}

For the worst case comparison, notice that if we set $\tau_i = \infty,
\forall i \neq a$, we obtain a lower bound for $\tau_a$ in each of the
constraints. Specifically, we get $\tau_a = \max_{v, a\in v^\star
  \ominus v } \frac{1}{\langle \mu, v^\star - v \rangle^2} >
\frac{1}{\langle \mu, v^\star - v_a \rangle^2}$, where $v_a$ is the
set that witness our complexity for arm $a$. Thus $\tau_a \ge
\frac{1}{D^2\Delta_a^2}$, so the worst case ratio between our bound
and their bound is $\frac{D^2}{\log|\Vcal|} \Lambda$. On the other
hand, it is easy to see that with $\tau_a = \frac{1}{\Delta_a^2}$,
using our definition of arm complexity, their program is feasible, and
so we have
\begin{align*}
\Low(\Vcal) \leq \sum_a \frac{1}{\Delta_a^2},
\end{align*}
which readily yields a lower bound on our complexity, in terms of theirs.

\subsection{Comparison with~\cite{gabillon2016improved}.}
In~\cite{gabillon2016improved}, the authors introduce a improved gap
by defining the complement of a set. Intuitively the complement is the
easiest set to compare with.  For any set $v \neq v^\star$, the gap is
\begin{align*}
\Delta^{(G)}_v = \max_{v', \langle\mu, v' - v \rangle > 0} \frac{\langle \mu, v' - v \rangle}{d(v',v)},
\end{align*}
and the set that achieves this maximum is the \emph{complement}
of $v$. A tie breaks in favor of the sets that are closer to $v$. The
gap of an arm $a$ is
$\Delta^{(G)}_a = \min_{v, a\in v\ominus v^\star} \Delta^{(G)}_v$,
and their sample complexity is
\begin{align*}
\order\left(\sum_a \frac{1}{(\Delta^{(G)}_a)^2} \log(K/\delta)\right),
\end{align*}
which is similar to the sample complexity
of~\cite{chen2014combinatorial} except the width is absorbed into the
new gap definition. As a consequence, this bound is never worse
than~\cite{chen2014combinatorial}.

In the homogeneous \biclique\xspace example, it is easy to see
$\Delta_a^{(G)} = \Delta$, since taking $v'=v^\star$ will always
achieve the maximum. Hence the bound becomes
$O(\frac{K}{\Delta^2}\log(K/\delta))$, which is $\Omega(\sqrt{s})$
worse than the bound in~\pref{thm:noninteractive_upper}.

For general $\mu$, note that for any set $v$, 
\begin{align*}
\Delta_v^{(G)} \geq \Delta_v =  \frac{\langle \mu, v^*-v \rangle}{d(v^*, v)} \geq \frac{\langle\mu, C_v-v \rangle}{d(C_v, v)} \frac{d(C_v,v)}{d(v^*, v)} \geq \Delta_v^{(G)} \frac{\Psi}{D},
\end{align*}
where $C_v$ is the complement of $v$. Hence the bounds satisfy
\begin{align*}
T_{\textrm{Gabillon16}} \cdot \Lambda \leq T \leq T_{\textrm{Gabillon16}} \cdot \frac{D^2}{\Psi^2}\Lambda.
\end{align*}


\section{Proofs}
\label{app:proofs}
In this section we provide the proofs of~\pref{thm:effdis}
and~\pref{thm:fixconsc}. Several lemmas and their proofs are provided
in~\pref{app:lemmas}.

\paragraph{Proof of~\pref{thm:effdis}.} 
We repeatedly use the following identity for the $\ell_1$ norm:
For any $u\in \{0, 1\}^K$ and any $x\in [0, 1]^K$, 
\begin{align}
\|x - u\|_1 = \langle x + u, \textup{\one} \rangle - 2\langle x,u \rangle. \label{eq:linl1}
\end{align}
This identity reveals that the disagreement region, $\Vcal_t$, is
polyhedral and hence Program~\pref{eq:dis_feasibility} is just a
linear feasibility problem. Now suppose that
Program~\pref{eq:dis_feasibility} is feasible and that
$x^\star \in \conv(\Vcal)$ is a feasible point. Then for every
distribution $p \in \Delta(\Vcal)$, $x^\star$ satisfies the linear
combination of the constraints weighted by $p \in \Delta(\Vcal)$,
which is precisely what we check by solving
Problem~\pref{eq:empfeassmall} and examining the objective value in~\pref{line:infeas} in each iteration of the algorithm. Hence by
contraposition, if the algorithm ever detects infeasibility, it must
be correct.

For the other direction, we use the regret bound for
Follow-the-Perturbed Leader~\citep{kalai2005efficient}. Succinctly,
when the learner makes decisions $d_t \in \Dcal \subset \RR^d$ and the
adversary chooses losses $\ell_t \in \Scal \subset \RR^d$, FTPL with
parameter $\epsilon \leq 1$ guarantees
\begin{align*}
\EE\sum_{t=1}^T\langle d_t, \ell_t\rangle - \min_{d \in \Dcal}
\sum_{t=1}^T\langle d, \ell_t\rangle \leq \epsilon RAT + D/\epsilon,
\end{align*}
where $D = \max_{d,d'\in \Dcal} \|d - d'\|_1, R = \max_{d \in
  \Dcal,\ell \in \Scal}|\langle d, s\rangle|$, and $A = \max_{\ell \in
  S}\|\ell\|_1$. Setting $\epsilon = \sqrt{D/(RAT)}$ gives
$2\sqrt{DRAT}$ regret. The algorithm chooses $d_t$ by sampling
$\sigma_{t} \sim \textrm{Unif}([0,1/\epsilon]^d)$ and playing $d_t =
\argmin_{d \in \Dcal} \langle d, \sigma_t +
\sum_{\tau=1}^{t-1}\ell_\tau \rangle$. This induces a distribution
over decisions $d_t$, which we denote by $p_t \in \RR^d$ and the
expectation accounts for this randomness. It will be important for us
that FTPL can accommodate adaptive adversaries, and hence the loss
$\ell_t$ can depend on $p_t$ but not on the random decision $d_t$.

In our case, we have $\Dcal = \conv(\Vcal)$, and we write $\ell_t =
\Delta \one - 2 \Delta x_t- \hat{\mu}$ where $x_t$ is the solution to
Program~\pref{eq:empfeassmall} in the $t^{\textrm{th}}$
iteration. This makes $D\le K$. Recall that $\hat{\mu}$ is
the empirical average of $y_t$. By Chernoff bound and a union bound,
with probability at least $1-\delta_1$, $\|\hat\mu - \bar\mu\|_1 < K
\sqrt{2\log(2K/\delta_1)}$.  
Since $\bar\mu \in [-1, 1]^K$, $x_t \in [0, 1]^K$,
$\Delta \in [0, 1]$, we get $A,R \le 5K\sqrt{\log(2K/\delta_1)}$ in our reduction.
So with $\epsilon = \sqrt{1 / (25KT\log(2K/\delta_1))}$ the regret is
upper bounded by $2K\sqrt{25KT\log(2K/\delta_1)}$.  Note that while $x_t$
and hence $\ell_t$ depends on the random choices of the learner
through $\tilde{p}_t$, we will actually apply the regret bound only on
the expectation, which we denote by $p_t$, which can be equivalently
viewed as the adversary sampling to generate $\tilde{p}_t$ and
$\ell_t$. To translate from $p_t$ to $\tilde{p}_t$ we need one final
lemma, which we prove in~\pref{app:lemmas}.
\begin{lemma}
\label{lem:ftpl_conc}
Let $p_t = \EE_{\sigma_t} \tilde{p}_t$ and let $\ell_t$ be any vector,
which may depend on $\tilde{p}_t$. Then with probability at least
$1-\delta$, simultaneously for all rounds $t \in [T]$
\begin{align*}
\left|\sum_{u \in \Vcal} (\tilde{p}_t(u) - p_t(u)) \langle u,\ell_t\rangle \right| \le 3K\sqrt{\frac{\log(2KT/\delta)}{2m}}.
\end{align*}
\end{lemma}

Now, 
we condition on the event in~\pref{lem:ftpl_conc} and use the fact that $x_t$ optimizes
Program~\pref{eq:empfeassmall} (which is defined by $\tilde{p}_t$)
and passes the check in~\pref{line:infeas}.
Applying the FTPL regret bound, we get
\begin{align*}
0 &\le \sum_{t=1}^T \frac{1}{m} \sum_{i=1}^{m} \Delta\langle x_t, \one-2u_{t,i} \rangle + \langle x_t, \hat{\mu} \rangle + \sum_{i=1}^{m} \langle u_{t,i},\Delta\one - \hat{\mu}\rangle\\
&= \sum_{t=1}^T \sum_{u\in \Vcal}\tilde{p}_t(u) \langle u, \ell_t \rangle + \langle x_t, \hat{\mu} + \one \Delta \rangle \\
& \le \sum_{t=1}^T\langle x_t, \Delta \one + \hat{\mu}\rangle + \min_u \sum_{t=1}^T \langle u, \ell_t\rangle + 2K\sqrt{25KT\log(2K/\delta_1)} + 3TK\sqrt{\frac{\log(2KT/\delta_2)}{2m}}.
\end{align*}

Note that we apply the regret bound on $p_t$, the expected decision of
the algorithm, rather than on $\tilde{p}_t$, the randomized
one. Setting $\delta_1=\delta_2=\delta/2$, dividing through by $T$, and using~\pref{eq:linl1},
we get
\begin{align*}
\forall u \in \Vcal, \langle u - \bar{x}, \hat{\mu}\rangle \le \Delta \|\bar{x}-u\|_1 + 2K\sqrt{25K\log(4K/\delta)/T}  + 3K\sqrt{\frac{\log(4KT/\delta)}{2m}},
\end{align*}
where $\bar{x}=\frac{1}{T}\sum_{t=1}^Tx_t$. The theorem follows by our choices for $T$ and $m$.
In particular, the number of oracle calls is $Tm = \order\rbr{ K^6/\Delta^{4}\cdot\log^2(K/\delta)\log(K\log(K/\Delta)/\delta)}.$

\paragraph{Proof of~\pref{thm:fixconsc}.}
The key lemma in the proof of~\pref{thm:fixconsc} is a uniform
concentration inequality on the empirical mean $\hat{\mu}$ used by the
algorithm. To state the inequality let $\bar{\mu}_t(a) \in \RR^d$ be
the conditional mean of $y_t(a)$, conditioning on all randomness up to
round $t$, including the execution of \Dis. This means that
$\bar{\mu}_t(a)$ is either $\mu(a)$ or $2\hat{v}_t(a)-1$, depending on
the outcome of the disagreement check. Recall the definition of
$\Delta_t$ in~\pref{alg:main}.  We first derive a
concentration inequality relating $\hat{\mu}_t$ to the empirical means
$\bar{\mu}_t$:


\begin{lemma}
\label{lem:uniform_martingale}
With the above definitions, for any $\delta \in (0,1)$, with
probability at least $1-\delta/2$
\begin{align*}
\forall t>0, \forall v\in \Vcal, \left|\frac{1}{t}\sum_{i=1}^t\langle v^\star - v, \bar{\mu}_i - y_i \rangle\right| \le d(v^\star, v) \Delta_t.
\end{align*}
\end{lemma}
This concentration inequality is not challenging to prove, but is much
sharper than ones used in prior work. The key difference is that our
inequality is a \emph{regret inequality} in the sense that it only
bounds differences with the true optimum $v^\star$, while the prior
results bound differences between all pairs of hypotheses. Our
definition of the version space $\Vcal(\hat{\mu},\Delta)$ enables
using this concentration inequality, which leads to our sample
complexity guarantees.

Define the event $\Ecal$ to be the event that~\pref{lem:uniform_martingale} holds and also that the
disagreement computation succeeds at all rounds for all arms, which by~\pref{thm:effdis} happens with probability
$1-\sum_{t>0} \frac{\delta}{t^2\pi^2} \ge 1-\delta/2$.  Under this
event, we establish two facts:
\begin{packed_enum}
\item $\forall t, v^\star \in \Vcal_t$ where $\Vcal_t =
  \Vcal(\hat{\mu}_t,\Delta_t)$ is the version space at round $t$ (\pref{lem:correct}). 
\item If $\Delta_t \le \Delta_a/3$, then arm $a$ will never be queried
  again (\pref{lem:sampcom}).
\end{packed_enum}
The correctness of the algorithm follows from the first fact. In
detail, the algorithm only terminates at round $t$ if for all arms,~\pref{alg:disagreement} detects infeasibility.
By~\pref{thm:effdis}, this means that
$\Vcal_t \cap \Vcal = \{\hat{v}_t\}$, and, by~\pref{lem:correct},
we must have $v^\star \in \Vcal_t$. Thus conditioned on $\Ecal$, the
algorithm returns $v^\star$.

For the sample complexity, from the second fact and the definition of
$\Delta_t$, arm $a$ will not be sampled once
\begin{align*}
t \ge \frac{72}{\Delta_a^2}\left(\Phi + \frac{\log(K\pi^2t^2/\delta)}{\Psi}\right).
\end{align*}
A sufficient condition for this transcendental inequality to hold is
(see~\pref{fact:transcendental}):
\begin{align*}
T_a \ge \frac{144}{\Delta_a^2}\left(\Phi + \frac{2\log(144/(\Delta_a^2\Psi)) + 2\log(K\pi^2/\delta)}{\Psi}\right).
\end{align*}
The sample complexity is at most $\sum_a T_a$, which proves the
theorem.

\section{Proofs for the lemmas}
\label{app:lemmas}

\begin{proof}[Proof of~\pref{lem:ftpl_conc}]
Let V be a $\RR^{K \times |\Vcal|} $ matrix whose columns are the
vectors $v \in \Vcal$. Recall that $p_t \in \Delta(\Vcal)$ is a
distribution over the perturbed leader at round $t$. Let $S_i \in
\{0,1\}^{|\Vcal|}$ be the indicator vector of the $i^{\textrm{th}}$
sample. Clearly, $\EE[S_i] = p_t$ and $ \hat{p}_t =\frac{1}{m}
\sum_{i=1}^{m} S_i $.  We have
\begin{align*}
\left| \sum_{u \in \Vcal}\left(\hat{p}_t(u)- p_t(u) \right) \langle u, \ell_t \rangle\right| = |\langle V\hat p_t - V p_t, \ell_t \rangle| \leq \|V\hat p_t - V p_t\|_{\infty} \|\ell_t\|_1.
\end{align*}
Let $()_j$ denote the $j$-th coordinate of a vector. By Hoeffding's
inequality and union bound we have
\begin{align*}
\PP\left[\forall t\in [T], \forall j \in [K],\  |(V\hat p_t)_j - (V p_t)_j| \ge \epsilon\right] \le 2KT \exp{(-2m\epsilon^2)},
\end{align*}
so that with probability at least $1 - \delta$
\begin{align*}
\forall t \in [T],\  \|V\hat p_t - V p_t\|_{\infty} \le \sqrt{\frac{\log{(2KT/\delta)}}{2m}}. 
\end{align*}
This proves the lemma. 
\end{proof}

\begin{proof}[Proof of~\pref{lem:uniform_martingale}]
Let $\Fcal_t$ be the $\sigma$-algebra conditioning on all randomness
up to and including the execution of \Dis\xspace for all arms $a \in
[K]$ at round $t$. Thus $y_t(a)$ is $\Fcal_t$ measurable and with $Z_t
= \sum_{i=1}^t (\bar \mu_i - y_i)$ it is not hard to see that
$\{Z_t\}_{t=1}^T$ forms a vector-valued martingale adapted to the
filtration $\{\Fcal_t\}_{t=0}^T$:
\begin{align*}
\EE[Z_t|\Fcal_t] &= \EE[(\bar{\mu}_t - y_t) + Z_{t-1}|Z_{t-1}]= Z_{t-1}.
\end{align*}
Observe also that $\bar{\mu}_t(a) - y_t(a)$ is a $0$-mean Subgaussian random variable with variance parameter at most $1$.
Thus, for any $v \neq v^\star$, Subgaussian martingale concentration gives
\begin{align*}
\PP\left[ \left|\sum_{a\in v\ominus v^\star} Z_t(a)/t\right| \geq \epsilon\right] \leq 2 \exp\left\{-\frac{t\epsilon^2}{8d(v,v^\star)}  \right\}.
\end{align*}
With a union bound, we get
\begin{align*}
\PP\left[\exists t, \exists v \in \Vcal, \left|\sum_{a\in v\ominus v^\star} Z_t(a)/t \right| \geq \epsilon_t(v,v^\star,\delta)\right] \leq 2 \sum_{t>0} \sum_{v\in\Vcal} \exp\left\{-\frac{t\epsilon_t(v,v^\star, \delta)^2}{8d(v,v^\star)}  \right\}.
\end{align*}
Following the argument in the proof of~\pref{thm:noninteractive_upper}, this right hand side will be
at most $\delta/2$ if
\begin{align*}
\epsilon_t(v^\star, v, \delta) = \sqrt{\frac{8d(v^\star, v)}{t}\log\left(\frac{Kt^2\pi^2|\Bcal(d(v^\star, v), v^\star)|}{\delta} \right)}.
\end{align*}
We set $\Delta_t = \sqrt{\frac{8}{t}\left(\Phi +
  \log(K\pi^2t^2/\delta)/\Psi\right)}$ so that for all $v \in \Vcal$, 
$\Delta_t d(v,v^\star) > \epsilon_t(v,v^\star,\delta)$, which
concludes the proof.
\end{proof}

\begin{lemma}\label{lem:correct}
Recall the definition of $\Vcal_t = \Vcal(\hat{\mu}_t, \Delta_t)$ at
round $t$, with $\Vcal(\hat{\mu},\Delta)$ defined
in~\pref{eq:version_space}. Then in event $\Ecal$, we have
that $\forall t, v^\star \in \Vcal_{t}$.
\end{lemma}
\begin{proof}
The proof is by induction. First, we know that if $v^\star \in
\Vcal_{t-1}$ then $\langle v^\star - v, \bar{\mu}_t \rangle \ge
\langle v^\star - v, \mu\rangle$. This follows since if arm $a$ is
queried then $\bar{\mu}_t(a) = \mu(a)$ and if arm $a$ is not
queried, we know that $v^\star(a) = \hat{v}_t(a)$ and our
hallucination sets $\bar{\mu}_t(a) = 2\hat{v}_t(a) - 1 = 2v^\star(a)-1$.
So if $v^\star(a) = 1$ and $v(a) = 0$, we have $\bar{\mu}_t(a) = 1 \ge \mu(a)$.
If $v^\star(a) = 0$ and $v(a) = 1$, we have $\bar{\mu}_t(a) = -1 \le \mu(a)$.
So in both cases we have 
$(v^\star(a) - v(a))\bar{\mu}_t(a) \ge (v^\star(a) - v(a))\mu(a)$.
Thus, if $\forall i \in[t-1], v^\star \in \Vcal_i$ (which is our
inductive hypothesis), then by~\pref{lem:uniform_martingale}
$\forall v\in \Vcal$
\begin{align*}
\langle v - v^\star, \hat{\mu}_t\rangle &\le \left\langle v - v^\star, \frac{1}{t}\sum_{i=1}^t \bar{\mu}_t\right\rangle + \Delta_t d(v,v^\star)
 \le \langle v - v^\star, \mu\rangle +  \Delta_t d(v,v^\star) \le \Delta_t d(v,v^\star).
\end{align*}
By definition of $\Vcal_t$, this proves that $v^\star \in
\Vcal_{t}$. Clearly the base case holds since $v^\star \in \Vcal_0 = \Vcal$.
\end{proof}

\begin{lemma}\label{lem:l1norm}
Let $x \in \textrm{conv}(\Vcal) = \sum_{i}\alpha_i v_i$, where $v_i
\in \Vcal$, $\sum_i \alpha_i = 1$, $\alpha_i \ge 0$ and let $v
\in \Vcal$. Then,
\begin{align}
\|x - v\|_1 = \sum_i \alpha_i\|v_i - v \|_1.
\end{align}
\end{lemma}
\begin{proof}
This follows by integrality of $v \in \Vcal$ and~\pref{eq:linl1}. In
particular, for integral $v$, $\|x - v\|$ is actually linear so we can
bring the $\sum_i \alpha_i$ outside the $\ell_1$ norm.
\end{proof}

\begin{lemma}\label{lem:sampcom}
Under event $\Ecal$, once $t$ is such that $\Delta_t < \Delta_a/3$,
arm $a$ will not be sample again.
\end{lemma}
\begin{proof}
We consider here the case where $v^\star(a) = 1$. For $v^\star(a) = 0$
the analysis is similar.  Assume for the sake of contradiction that
$a$ is sampled, which means that \textsc{Dis}$(a,1 - \hat v_t(a),
\Delta_t, \hat\mu_t)$ returns \true. If $v^\star(a) = 1$, then
$\forall v\in \Vcal$ with $v(a) = 0$ we have
\begin{align*}
 \langle v^\star - v, \hat{\mu}_{t}\rangle &\ge \left\langle v^\star - v, \frac{1}{t}\sum_{\tau=1}^{t} \bar{\mu}_\tau\right\rangle - \Delta_{t}d(v^\star,v)
 \ge \langle v^\star - v, \mu \rangle - \Delta_{t}d(v^\star,v)
 > 2d(v^\star,v)\Delta_t
\end{align*}
The first inequality is~\pref{lem:uniform_martingale}, the second
uses the property of the hallucinated samples that we used in~\pref{lem:correct}. The last inequality is due to $\Delta_a \le
\frac{\langle v^\star - v, \mu \rangle}{d(v^\star, v)}$ and our
assumption that $\Delta_t < \Delta_a/3$. This implies that
$\hat{v}_t(a) = 1$, which means that we execute \Dis\xspace to check
if any surviving hypothesis $v \in \Vcal_t$ has $v(a) = 0$. Since we
sampled arm $a$, this means there exists $x \in \conv(\Vcal)$ such that
\begin{align*}
\forall u \in \Vcal \langle u - x, \hat{\mu}_t\rangle \le \Delta_t \|u
- x\|_1 + \Delta_t.
\end{align*}
This follows by~\pref{thm:effdis} which holds under the event
$\Ecal$. Now write $x = \sum_i \alpha_i v_i$ where $\alpha$ is a
distribution and $v_i \in \Vcal$. Since $x(a) = 0$, we must have
$v_i(a) = 0$ for all $i$. This means that
\begin{align*}
 \langle v^\star - x, \hat{\mu}_{t}\rangle &= \sum_i \alpha_i \langle v^\star - v_i, \hat \mu_t \rangle 
> 2\sum_i \alpha_i d(v^\star,v_i)\Delta_t
\ge \Delta_t\|v^\star - x\|_1 + \Delta_t.
\end{align*}
The last inequality is due to~\pref{lem:l1norm} and the fact that
$\forall i, d(v^\star, v_i) \ge 1$. This contradicts the guarantee in~\pref{thm:effdis}, which means that \textsc{Dis}$(a,1 - \hat
v_t(a), \Delta_t, \hat\mu_t)$ cannot return \true.
\end{proof}

We use the Lemma 8 from~\cite{antos2010active}.
\begin{fact}[Lemma 8 from~\cite{antos2010active}]
\label{fact:transcendental}
Let $a>0$, for any $t \ge \frac{2}{a} \max\{ (\log \frac{1}{a} - b), 0\}$, we have $at+b \ge \log t$. 
\end{fact}

\section{Proof of~\pref{thm:fixed_budget}}\label{app:fixed_budget}

In the fixed budget setting, we follow a classic rejection strategy
used by many algorithms in other settings (e.g., Successive
Rejects~\cite{audibert2010best}, SAR~\cite{bubeck2013multiple},
CSAR~\cite{chen2014combinatorial} and also the algorithm
of~\cite{gabillon2016improved}).

We require several new definitions.  First recall that our definition
of the gap for arm $a$ is $\Delta_a$. Let $\Delta^{(j)}$ be the
$j^\textrm{th}$ largest element in $\{\Delta_a\}_{a \in [K]}$. Then
the main complexity measure is $\tilde{H} = \max_j (K+1-j)
(\Delta^{(j)})^{-2}$.
For short hand we define the partial
harmonic sum $\tlog(t) = \sum_{i=1}^t 1/i$. Assume that the total
budget is $T$, and define
\begin{align*}
  n_t = \left\lceil \frac{T-K}{\tlog(K)(K+1-t)}\right\rceil, \qquad n_0 = 0
\end{align*}
which will be related to the number of queries issued in each round of
our algorithm. As before, let $\hat{\mu}_t$ be the empirical mean at
round $t$ of the algorithm and let $\hat{v}_t = \argmax_{v \in \Vcal}
\langle v, \hat{\mu}_t\rangle$ be the empirical maximizer. Define the
empirical gaps at round $t$ for hypotheses and arms respectively as
\begin{align*}
\hat{\Delta}_{t,v} = \frac{\langle \hat{\mu}_t, \hat{v}_t - v\rangle}{d(\hat{v}_t, v)}, \qquad \hat{\Delta}_{t,a} = \min_{a \in \hat{v}_t \ominus v} \hat{\Delta}_{t,v}.
\end{align*}

\begin{algorithm}[t]
\caption{Fixed budget algorithm for combinatorial identification}\label{alg:fixbudget}
\begin{algorithmic}[1]
  \State Input: $\Vcal$, set of arm $[K], \{n_t\}_d$
   \State Set $t \gets 1, A_1 \gets \emptyset, R_1\gets \emptyset$
   \For{$t = 1,2,3,...,K$} 
  \State Sample arms in $[K] \backslash (A_t\cup R_t)$ for $n_t - n_{t-1}$ times. For $a \in A_t$ use sample value $1$ and for $a \in R_t$ use sample value $-1$ (i.e., hallucinate samples). 
  \State Update $\hat \mu_t$ and find $\hat v_t = \arg\max_{v \in \Vcal} \langle v, \hat \mu_t \rangle$.
  \State $\hat{a}_t = \argmax_{a \in [K] \backslash (A_t \cup R_t)} \hat{\Delta}_{t,a}$. \label{line:gap_end}
  \State If $\hat{a}_t \in \hat{v}_t$, then $A_{t+1} = A_t \cup\{\hat{a}_t\}, R_{t+1} = R_t$, else $A_{t+1} = A_t, R_{t+1} = R_t\cup \{\hat{a}_t\}$.
  \EndFor
  \State \Return $A_{K+1}$
\end{algorithmic}
\end{algorithm}

With these definitions, we are now ready to describe the fixed budget
algorithms, with pseudocode in~\pref{alg:fixbudget}. The algorithm
maintains a set of ``accepted" and ``rejected" arms, $A_t$ and $R_t$
in the pseudocode at round $t$, and once an arm is marked ``accept" or
``reject" it is never queried again. At each round $t$ we issue
several queries to all surviving arms, ensuring that each arm has
$n_t$ total queries, and then we find the arm with the largest
empirical gap $\hat{\Delta}_{t,a}$ and accept it if it is included in
the ERM $\hat{v}_t$. Otherwise we reject.  Note that the algorithm is
not oracle efficient, since computing the empirical arm gaps is not
amenable to linear optimization.

\begin{proof}[Proof of~\pref{thm:fixed_budget}]
First, note that in each round we eliminate one arm and sample the
rest for $n_t-n_{t-1}$ times. Thus after round $t$ we have sampled
each surviving arm $n_t$ times, and exactly one arm is sampled $n_i$
times for each $i \in [K]$. Thus the total number of samples is
\begin{align*}
\sum_{t=1}^{K} n_t  &= \sum_{t=1}^{K} \left\lceil \frac{T-K}{\tlog(K) (K+1-t)}\right\rceil \leq \sum_{t=1}^K \frac{T-K}{\tlog(K)(K+1-t)} + 1 = T.
\end{align*}
Second, define $\bar{\mu}_t$ as before to be the mean of the
  all samples up to and including round $t$, taking into account the
  hallucination. $\bar{\mu}_t(a)$ is an average of $n_t$ terms where
  if at round $i \le t$ we place $a \in A_i$, then the last $n_t -
  n_i$ terms are just $1$. Similarly if at round $n_i$ we place $a \in
  R_i$ then the last $n_t - n_i$ terms are $-1$. Otherwise all terms
  are simply $\mu(a)$. Formally,
  \begin{align*}
    \bar{\mu}_t(a) = \frac{1}{n_t} \sum_{\tau=1}^K (n_\tau - n_{\tau-1}) \left[\mu \one\{a \notin R_\tau \cup A_\tau\} + \one\{a \in A_\tau\} - \one\{a \in R_\tau\}\right] .
  \end{align*}
  Note that this is different but related to our definition in the
  fixed confidence proof. We define the high probability event:
\begin{align*}
\Ecal \triangleq \{ \forall t\in [K], \forall v \in \Vcal, |\langle v-v^\star,\hat{\mu}_t-\bar{\mu}_t\rangle| < c d(v,v^\star)\Delta^{(t)} \},
\end{align*}
where $c<1$ is a constant that we will set later. Now we show that
$\Ecal$ holds with high probability:
\begin{align*}
\PP[\bar{\Ecal}] &\leq \sum_t \sum_{v\in \Vcal} \exp\left\{-\frac{c^2d(v,v^\star) (\Delta^{(t)})^2(T-K)}{\tilde\log(K)(K+1-t)}\right\}
\leq K \sum_{v\in \Vcal} \exp\left\{-\frac{c^2(T-K)d(v,v^\star)}{\tilde\log(K)\tilde{H}}\right\}\\
&\leq K \sum_{k\in [K]} \exp\left\{-\frac{c^2(T-K)k}{\tilde\log(K)\tilde{H}} + \log|\Bcal(k, v^\star)|\right\}
\leq K^2 \exp \left\{\Psi \left(\Phi -  \frac{(T-K) c^2}{\tilde{\log}(K) \tilde{H}}\right)  \right\}.
\end{align*}
We proceed to show that, conditioned on
event $\Ecal$, $A_{K+1} = v^\star$. At round $t$, define
\begin{align*}
a^\star_t = \argmax_{a \in [K] \setminus (A_t \cup R_t)} \Delta_a,
\end{align*}
where $A_t$ and $R_t$ are the accepted and reject arms at
the beginning of round $t$ and $\Delta_a$ is the true arm
complexity. Further assuming (inductively) that $A_t \subset v^\star$
and $R_t \cup v^\star = \emptyset$, we establish five facts:

\noindent \textbf{Fact 1.} At the beginning of round $t$, $a^\star_t$ satisfies
$\Delta_{a^\star_t} \ge \Delta^{(t)}$. If this statement does not hold at
round $t$, then we must have eliminated all of the the $t$ arms
$\Delta^{(1)}, \ldots, \Delta^{(t)}$. However, since we eliminate
exactly one arm in each round, we can only eliminate $t-1$ arms before
round $t$, which produces a contradiction since $a^\star_t$ is the
maximizer.

\noindent \textbf{Fact 2.} Under the inductive hypothesis, for all $v
\in \Vcal$, we have $\langle\bar{\mu}_t, v^\star - v\rangle \ge
\langle \mu, v^\star - v\rangle$. This is similar to the argument we
used in the fixed confidence proof. For any arm $a$, if $a \notin A_t
\cup R_t$ then the corresponding terms are equal. If $a \in A_t$ then
since by induction we know $a \in v^\star$, the term for $v^\star$ is
as high as possible and analogously if $a \in R_t$ the term for
$v^\star$ is as low as possible.

\noindent \textbf{Fact 3.} $a^\star_t \in \hat{v}_t \iff a^\star_t \in
v^\star$. Assume for the sake of contradiction that $a^\star_t\in
\hat{v}_t$ and $a^\star_t \notin v^\star$. The proof is the same for
the other case. We have
\begin{align*}
\Delta_{a^\star_t} = \min_{a^\star_t \in v^\star \ominus v} \frac{\langle \mu, v^\star -v \rangle}{d(v^\star, v)} \le \frac{\langle \mu, v^\star - \hat{v}_t \rangle}{d( v^\star, \hat{v}_t)}.
\end{align*}
Thus we have $\frac{\langle \mu, \hat{v}_t -
  v^\star\rangle}{d(v^\star, \hat{v}_t)} \leq - \Delta_{a^\star_t}$. By the
previous fact we know $\Delta_{a_t} \geq \Delta^{(t)}$ since $a^\star_t$ is
the maximizer.  Now, conditioned on $\Ecal$:
\begin{align*}
\frac{\langle \hat{\mu}_t, \hat{v}_t - v^\star \rangle}{d(\hat{v}_t, v^\star)} < \frac{\langle \bar{\mu}_t, \hat{v}_t - v^\star \rangle}{d(\hat{v}_t , v^\star)} + c\Delta^{(t)} \leq \frac{\langle \mu, \hat{v}_t - v^\star\rangle}{d(\hat{v}_t, v^\star)} + c\Delta^{(t)} < \Delta^{(t)} - \Delta_{a^\star_t} \le 0.
\end{align*}
The first inequality is by event $\Ecal$, the second is by Fact 2 and
the final one is by Fact 1 and the definition of $a_t^\star$. This results
in a contradiction.

\noindent \textbf{Fact 4.} Let $\tilde{v}_{t,a^\star_t}$ be the set that
witnesses $\hat{\Delta}_{t,a^\star_t}$, i.e. $\tilde{v}_{t,a^\star_t}
= \argmin_{v: a^\star_t \in v\ominus \hat{v}_t} \hat{\Delta}_{t,v}$. We
have that $\langle \hat{\mu}_t, v^\star - \tilde{v}_{t,a^\star_t}
\rangle > 0$. To see why, note that $a^\star_t \in \hat{v}_t \ominus
\tilde{v}_{t,a^\star_t}$ and by Fact 3 we have $a^\star_t \in v^\star
\ominus \tilde{v}_{t,a^\star_t}$. Conditioning on $\Ecal$ and using the
fact that the true gap $\Delta_{a^\star_t}$ involves minimizing over
$v \in \Vcal$ we get
\begin{align*}
\frac{\langle \hat{\mu}_t, v^\star - \tilde{v}_{t,a^\star_t} \rangle}{d(v^\star, \tilde{v}_{t,a^\star_t})} \ge \frac{\langle \bar{\mu}_t, v^\star - \tilde{v}_{t,a^\star_t} \rangle}{d(v^\star, \tilde{v}_{t,a^\star_t})} -c \Delta^{(t)} \geq
\frac{\langle \mu, v^\star - \tilde{v}_{t,a^\star_t} \rangle}{d(v^\star, \tilde{v}_{t,a^\star_t})} - c\Delta^{(t)} > \Delta_{a^\star_t} - \Delta^{(t)} \ge 0.
\end{align*}
The last step here uses Fact 1. 

\noindent \textbf{Fact 5.} $\hat{a}_t \in \hat{v}_t \iff \hat{a}_t \in
v^\star$. Assume for the sake of contradiction that $\hat{a}_t
\in \hat{v}_t, \hat{a}_t \notin v^\star$. We have
\begin{align*}
\hat{\Delta}_{t,\hat{a}_t} = \min_{\hat{a}_t \in \hat{v}_t \Delta v} \frac{\langle \hat{\mu}_t, \hat{v}_t -v \rangle}{d( \hat{v}_t, v)} \le \frac{\langle \hat{\mu}_t, \hat{v}_t - v^\star \rangle}{d(\hat{v}_t, v^\star)}.
\end{align*}
As above, let $\tilde{v}_{t,a^\star_t}$ be the set that witnesses
$\hat{\Delta}_{t,a^\star_t}$. Since $\hat{a}_t$ maximizes
$\hat{\Delta}_{t,a}$ over all surviving arms $a$ and since $a^\star_t$ is
surviving by definition, we have
\begin{align*}
\hat{\Delta}_{t,\hat{a}_t} &\ge \hat{\Delta}_{t,a^\star_t} = \min_{a^\star_t \in \hat{v}_t \ominus v} \frac{\langle \hat{\mu}_t, \hat{v}_t -v \rangle}{d( \hat{v}_t, v)} = \frac{\langle \hat{\mu}_t, \hat{v}_t -\tilde{v}_{t,a^\star_t} \rangle}{d( \hat{v}_t, \tilde{v}_{t,a^\star_t})} \\
&\ge  \frac{\langle \hat{\mu}_t, \hat{v}_t - v^\star \rangle + \langle \hat{\mu}_t, v^\star -\tilde{v}_{t,a^\star_t} \rangle}{d( \hat{v}_t, v^\star) + d(v^\star, \tilde{v}_{t,a^\star_t})} 
\ge \min\left\{\frac{\langle \hat{\mu}_t, \hat{v}_t - v^\star \rangle}{d( \hat{v}_t, v^\star)}, \frac{\langle \hat{\mu}_t, v^\star -\tilde{v}_{t,a^\star_t} \rangle}{d(v^\star, \tilde{v}_{t,a^\star_t})} \right\}\\
&\triangleq \min\{a,b\}.
\end{align*}
The last inequality holds since both terms in the numerator are
non-negative as we have shown above in Fact 4.  Since we previously
upper bounded $\hat{\Delta}_{t,\hat{a}_t}$ by what we are now calling
$a$, we have $a\ge \min\{a, b\}$. If $a \le b$, then all of the
inequalities are actually equalities, so we must have $a = b$. The
other case is that $a > b$, so we can address both cases by
considering $a \geq b$. Expanding the definition and applying the
concentration inequality, we have
\begin{align*}
 b \triangleq \frac{\langle \hat{\mu}_t, v^\star -\tilde{v}_{t,a^\star_t} \rangle}{d(v^\star, \tilde{v}_{t,a^\star_t})} \ge \frac{\langle \mu, v^\star -\tilde{v}_{t,a^\star_t} \rangle}{d(v^\star, \tilde{v}_{t,a^\star_t})} - c\Delta^{(t)} \ge \Delta_{a^\star_t} - c\Delta^{(t)}.
\end{align*}
On the other hand,
\begin{align*}
a \triangleq \frac{\langle \hat{\mu}_t, \hat{v}_t - v^\star \rangle}{d( \hat{v}_t, v^\star)} \le \frac{\langle \mu, \hat{v}_t - v^\star \rangle}{d( \hat{v}_t, v^\star)} + c\Delta^{(t)} \le c\Delta^{(t)}.
\end{align*}
Both of these calculations also require Fact 2. Setting $c = 1/3$, we
have
\begin{align*}
\Delta_{a^\star_t} \le 2c \Delta^{(t)} < \Delta^{(t)},
\end{align*}
which contradicts Fact 1 and the definition of $a_t$.

\textbf{Wrapping up.} To conclude the proof, we proceed by
induction. Clearly the base case that $A_0 \subset v^\star$ and $R_0
\cap v^\star = \emptyset$ is true. Now conditioning on $\Ecal$ and
assuming the inductive hypothesis, we have that by Fact 5, the arm
$\hat{a}_t \in \hat{v}_t \iff \hat{a}_t \in v^\star$. This directly
proves the inductive step since the algorithm's rule for accepting an
arm agrees with $v^\star$.
\end{proof}

\section{Proof of~\pref{thm:refinefixcon}}\label{app:refined_proof}
Recall that in the main concentration argument in~\pref{lem:uniform_martingale}, we proved that
\begin{align*}
\PP\big[\exists t \in \NN, \exists v \in \Vcal, |\langle v^\star - v, \hat{\mu}_t - \frac{1}{t}\sum_{i=1}^{t}\bar{\mu}_i \rangle| \ge \epsilon_t(v,v^\star, \delta)\big] \leq 2 \sum_{t\in\NN} \sum_{v \in \Vcal} \exp\left\{-\frac{t\epsilon_t(v,v^\star, \delta)^2}{8d(v,v^\star)}\right\}.
\end{align*}
Setting
\begin{align*}
\epsilon_t (v,v^\star, \delta) = \sqrt{ \frac{8d(v,v^\star)}{t} \log \frac{|\Bcal(d(v,v^\star), v^\star)|\pi^2 K t^2}{3\delta} },
\end{align*}
we have that the probability of this event is at most
$\delta$. Previously we set each hypothesis to have the same
confidence interval $\Delta_t$ which provided an upper bound on
$\epsilon_t(v,v^\star,\delta)$ for all $v$.
This enabled us to write the disagreement region as a polyhedral set
in $\Vcal$, but to obtain a more refined bound, we would like to use
$\epsilon_t (v,v^\star, \delta)$ directly. However, note that
$\epsilon_t (v,v', \delta) \neq \epsilon_t (v',v, \delta)$ unless the
hypothesis space is symmetric. We symmetrize $\epsilon_t$ by defining
\begin{align*}
D(v, v') \triangleq \max\{\log|\Bcal(d(v,v'), v)|, \log|\Bcal(d(v,v'), v')|\} = D(v', v),
\end{align*}
and the symmetric confidence interval
\begin{align}
\epsilon'_t (v,v', \delta) \triangleq \sqrt{ \frac{8d(v,v')}{t} \left(\log \frac{\pi^2K t^2}{3\delta}  + D(v, v')\right)}. \label{eq:refined_epsilon}
\end{align}
Define the hypothesis complexity measures, for $v \ne v^\star$
\begin{align*}
H^{(1)}_v = \frac{d(v,v^\star)}{\langle \mu, v^\star - v\rangle^2}, & \qquad H^{(2)}_v = \frac{d(v,v^\star) D(v,v^\star)}{\langle \mu, v^\star - v\rangle^2}.
\end{align*}
The arm complexity measures, defined previously, are $H^{(1)}_a =
\max_{a \in v\ominus v^\star} H^{(1)}_v$ and $H^{(2)}_a = \max_{v: a
  \in v \ominus v^\star} H^{(2)}_v$.  The main difference here is that
we are not normalizing by $d(v,v^\star)^2$ as we did in the proof
of~\pref{thm:fixconsc} but rather just $d(v,v^\star)$. In some sense
we replace the term depending on $\Psi$ with $H^{(1)}_a$ and the term
depending on $\Phi$ with $H^{(2)}_a$.

To prove~\pref{thm:refinefixcon}, we construct an inefficient fixed
confidence algorithm, with pseudocode in~\pref{alg:inefffixcon}. The
algorithm is essentially identical to~\pref{alg:main}, except we use
the new definition $\epsilon'$ in the confidence bounds defining the
version space, which forces us to do explicit enumeration. One other
minor difference is that we are now explicitly enforcing monotonicity
of the version space, so we need not use hallucination as we did
before. We now turn to the proof. 

\begin{algorithm}[t]
\caption{Inefficient fixed confidence algorithm}
\label{alg:inefffixcon}
  \begin{algorithmic}[1]
    \State Input: $\Vcal$, set of arms $[K]$, $\delta$
    \State Set $\Vcal_1 = \Vcal$
    \For{$t = 1, 2, 3,\dots$}
    \State $\Acal_t = \emptyset$
    \For{$a\in[K]$}
      \If{$\exists v,v' \in \Vcal_t$ such that $v(a) \ne v'(a)$}
      \State $\Acal_t = \Acal_t \cup a$, query $a$, set $y_t(a) \sim \Ncal(\mu(a),1)$ 
      \EndIf
    \EndFor
      \State Update $\hat{\mu}_t = \frac{1}{t} \sum_{\tau=1}^ty_t$.
      \State $\Rcal_t \gets \{v \in \Vcal_t \mid \exists u \in \Vcal_t, u \ne v, \langle u-v, \hat{\mu}_t \rangle > \epsilon_t'(u, v,\delta)\}$
      \State Update $\Vcal_{t+1} \gets \Vcal_{t} \setminus \Rcal_{t}$
      \State If $|\Vcal_{t+1}| = 1$ \Return the single element $v \in \Vcal_{t+1}$.
    \EndFor
  \end{algorithmic}
\end{algorithm}

\begin{proof}[Proof of~\pref{thm:refinefixcon}]
  In a similar way to~\pref{lem:uniform_martingale} we can prove that
  \begin{align*}
    \PP\left[\forall t, \forall v \in \Vcal_t, |\langle v^\star - v, \hat{\mu}_t - \mu\rangle| > \epsilon_t'(v^\star,v,\delta)\right] \le 2 \sum_t \sum_{v \in \Vcal}\exp\left\{ - \frac{t\epsilon_t'(v,v^\star,\delta)^2}{8d(v,v^\star)}\right\}.
  \end{align*}
  The important thing here is that if $v \in \Vcal_t$ then we must
  query every $a \in v \ominus v^\star$ and moreover since we are
  explicitly enforcing monotonicity (i.e. $\Vcal_{t} \subset
  \Vcal_{t-1}$), we also queried all of these arms in all previous
  rounds. Thus we are obtaining unbiased samples to evaluate these
  mean differences. Using the definition of $\epsilon_t'$
  in~\pref{eq:refined_epsilon}, this probability is at most $\delta$.

Next we prove that when the algorithm terminates, the output is
$v^\star$. We work conditional on the $1-\delta$ event that the
concentration inequality holds. We argue that $v^\star$ is never
eliminated, or formally $v^\star \notin \Rcal_t$ for all $t$. To see
why observe that $\forall v \in \Vcal_{t-1} \neq v^\star $, we have
\begin{align*}
\langle \hat{\mu}_{t}, v-v^\star \rangle \le \langle \mu, v - v^\star \rangle + \epsilon'_t(v, v^\star, \delta).
\end{align*}
This means that no surviving $v \in \Vcal_t$ can eliminate
$v^\star$. This verifies correctness of the algorithm, since $v^\star$
is never eliminated, so it must be the single element in $\Vcal_t$
when the algorithm terminates.

We now turn to the sample complexity. We argue here that if $t > 32
H^{(1)}_a\log(\pi^2Kt^2/(3\delta)) + 32 H^{(2)}_a$ then from round $t$
onwards, arm $a$ will not be sampled again. This condition on $t$
implies that for all $v \in \Vcal$ such that $a \in v \ominus v^\star$,
we have
\begin{align*}
  \epsilon'_{t}(v^\star, v, \delta)&\triangleq \sqrt{\frac{8 d(v,v^\star)}{t}\left(\log(\pi^2Kt^2/(3\delta)) + D(v,v^\star)\right)} < \langle \mu, v^\star - v\rangle/2,
\end{align*}
by the definitions of $H^{(1)}_a$ and $H^{(2)}_a$. Using this simpler
fact we argue that $a$ cannot be sampled again. Working toward a
contradiction, assume that $a$ is sampled at round $t+1$, which means
there exists two hypotheses $v_1,v_2 \in \Vcal_{t+1}$ such that
$v_1(a) \ne v_2(a)$. Since $v^\star \in \Vcal_{t+1}$ this implies that
there exists $v \in \Vcal_{t+1}$ such that $v^\star(a) \ne v(a)$. But we clearly have
\begin{align*}
\langle v^\star - v, \hat \mu_{t} \rangle \ge \langle v^\star - v, \mu \rangle - \epsilon'_{t}( v^\star, v, \delta) > \epsilon'_{t}( v^\star, v, \delta)
\end{align*}
which is a contradiction since $v$ must have been eliminated at round
$t$. This proves that $a \notin \Acal_{t+1}$, and since $\Vcal_{t+1}$
is monotonically shrinking, so is $\Acal_t$, which means that $a$ is never sampled again.





To summarize, we have now shown that for each arm $a$, the arm will be
sampled at most $t_a$ times, where $t_a$ is the smallest integer
satisfying
\begin{align*}
t_a \ge 32 H^{(1)}_a \log(\pi^2Kt_a^2/(3\delta)) + 32 H^{(2)}_a.
\end{align*}
The final result now follows from an application of~\pref{fact:transcendental}.
\end{proof}

\bibliographystyle{plainnat}
\bibliography{bibliography}

\begin{thebibliography}{39}
\providecommand{\natexlab}[1]{#1}
\providecommand{\url}[1]{\texttt{#1}}
\expandafter\ifx\csname urlstyle\endcsname\relax
  \providecommand{\doi}[1]{doi: #1}\else
  \providecommand{\doi}{doi: \begingroup \urlstyle{rm}\Url}\fi

\bibitem[Abbe et~al.(2016)Abbe, Bandeira, and Hall]{abbe2016exact}
Emmanuel Abbe, Afonso~S Bandeira, and Georgina Hall.
\newblock Exact recovery in the stochastic block model.
\newblock \emph{IEEE Transactions on Information Theory}, 2016.

\bibitem[Addario-Berry et~al.(2010)Addario-Berry, Broutin, Devroye, and
  Lugosi]{addarioberry2010combinatorial}
Louigi Addario-Berry, Nicolas Broutin, Luc Devroye, and G\'{a}bor Lugosi.
\newblock {On combinatorial testing problems}.
\newblock \emph{The Annals of Statistics}, 2010.

\bibitem[Agarwal et~al.(2014)Agarwal, Hsu, Kale, Langford, Li, and
  Schapire]{agarwal2014taming}
Alekh Agarwal, Daniel Hsu, Satyen Kale, John Langford, Lihong Li, and Robert~E.
  Schapire.
\newblock Taming the monster: A fast and simple algorithm for contextual
  bandits.
\newblock In \emph{International Conference on Machine Learning}, 2014.

\bibitem[Antos et~al.(2010)Antos, Grover, and Szepesv{\'a}ri]{antos2010active}
Andr{\'a}s Antos, Varun Grover, and Csaba Szepesv{\'a}ri.
\newblock Active learning in heteroscedastic noise.
\newblock \emph{Theoretical Computer Science}, 2010.

\bibitem[Arias-Castro and Cand\`{e}s(2008)]{ariascastro2008searching}
Ery Arias-Castro and Emmanuel~J. Cand\`{e}s.
\newblock {Searching for a trail of evidence in a maze}.
\newblock \emph{The Annals of Statistics}, 2008.

\bibitem[Audibert and Bubeck(2010)]{audibert2010best}
Jean-Yves Audibert and S{\'e}bastien Bubeck.
\newblock Best arm identification in multi-armed bandits.
\newblock In \emph{Conference on Learning Theory}, 2010.

\bibitem[Balakrishnan et~al.(2011)Balakrishnan, Xu, Krishnamurthy, and
  Singh]{balakrishnan2011noise}
Sivaraman Balakrishnan, Min Xu, Akshay Krishnamurthy, and Aarti Singh.
\newblock {Noise Thresholds for Spectral Clustering}.
\newblock In \emph{Advances in Neural Information Processing Systems}, 2011.

\bibitem[Bubeck et~al.(2013)Bubeck, Wang, and Viswanathan]{bubeck2013multiple}
S{\'e}bastian Bubeck, Tengyao Wang, and Nitin Viswanathan.
\newblock Multiple identifications in multi-armed bandits.
\newblock In \emph{International Conference on Machine Learning}, 2013.

\bibitem[Butucea and Ingster(2013)]{butucea2013detection}
Cristina Butucea and Yuri~I. Ingster.
\newblock {Detection of a sparse submatrix of a high-dimensional noisy matrix}.
\newblock \emph{Bernoulli}, 2013.

\bibitem[Carpentier and Locatelli(2016)]{carpentier2016tight}
Alexandra Carpentier and Andrea Locatelli.
\newblock Tight (lower) bounds for the fixed budget best arm identification
  bandit problem.
\newblock In \emph{Conference on Learning Theory}, 2016.

\bibitem[Castro and T{\'{a}}nczos(2015)]{tanczos2013adaptive}
Rui Castro and Ervin T{\'{a}}nczos.
\newblock {Adaptive sensing for estimation of structured sparse signals}.
\newblock \emph{IEEE Transactions on Information Theory}, 2015.

\bibitem[Chen et~al.(2016)Chen, Gupta, and Li]{chen2016pure}
Lijie Chen, Anupam Gupta, and Jian Li.
\newblock Pure exploration of multi-armed bandit under matroid constraints.
\newblock In \emph{Conference on Learning Theory}, 2016.

\bibitem[Chen et~al.(2017)Chen, Gupta, Li, Qiao, and Wang]{chen2017nearly}
Lijie Chen, Anupam Gupta, Jian Li, Mingda Qiao, and Ruosong Wang.
\newblock Nearly optimal sampling algorithms for combinatorial pure
  exploration.
\newblock In \emph{Conference on Learning Theory}, 2017.

\bibitem[Chen et~al.(2014)Chen, Lin, King, Lyu, and
  Chen]{chen2014combinatorial}
Shouyuan Chen, Tian Lin, Irwin King, Michael~R Lyu, and Wei Chen.
\newblock Combinatorial pure exploration of multi-armed bandits.
\newblock In \emph{Advances in Neural Information Processing Systems}, 2014.

\bibitem[Chen and Xu(2016)]{chen2014statistical}
Yudong Chen and Jiaming Xu.
\newblock {Statistical-computational tradeoffs in planted problems and
  submatrix localization with a growing number of clusters and submatrices}.
\newblock \emph{Journal of Machine Learning Research}, 2016.

\bibitem[Cohn et~al.(1994)Cohn, Atlas, and Ladner]{cohn1994improving}
David Cohn, Les Atlas, and Richard Ladner.
\newblock Improving generalization with active learning.
\newblock \emph{Machine Learning}, 1994.

\bibitem[Dasgupta et~al.(2007)Dasgupta, Hsu, and
  Monteleoni]{dasgupta2007general}
Sanjoy Dasgupta, Daniel Hsu, and Claire Monteleoni.
\newblock A general agnostic active learning algorithm.
\newblock In \emph{Advances in Neural Information Processing Systems}, 2007.

\bibitem[Even-Dar et~al.(2006)Even-Dar, Mannor, and Mansour]{even2006action}
Eyal Even-Dar, Shie Mannor, and Yishay Mansour.
\newblock Action elimination and stopping conditions for the multi-armed bandit
  and reinforcement learning problems.
\newblock \emph{Journal of Machine Learning Research}, 2006.

\bibitem[Gabillon et~al.(2016)Gabillon, Lazaric, Ghavamzadeh, Ortner, and
  Bartlett]{gabillon2016improved}
Victor Gabillon, Alessandro Lazaric, Mohammad Ghavamzadeh, Ronald Ortner, and
  Peter Bartlett.
\newblock Improved learning complexity in combinatorial pure exploration
  bandits.
\newblock In \emph{Artificial Intelligence and Statistics}, 2016.

\bibitem[Garivier and Kaufmann(2016)]{garivier2016optimal}
Aur{\'e}lien Garivier and Emilie Kaufmann.
\newblock Optimal best arm identification with fixed confidence.
\newblock In \emph{Conference on Learning Theory}, 2016.

\bibitem[Hanneke(2014)]{hanneke2014theory}
Steve Hanneke.
\newblock Theory of disagreement-based active learning.
\newblock \emph{Foundations and Trends in Machine Learning}, 2014.

\bibitem[Hsu(2010)]{hsu2010algorithms}
Daniel Hsu.
\newblock \emph{Algorithms for Active Learning}.
\newblock PhD thesis, University of California at San Diego, 2010.

\bibitem[Huang et~al.(2015)Huang, Agarwal, Hsu, Langford, and
  Schapire]{huang2015efficient}
Tzu-Kuo Huang, Alekh Agarwal, Daniel Hsu, John Langford, and Robert~E.
  Schapire.
\newblock Efficient and parsimonious agnostic active learning.
\newblock In \emph{Advances in Neural Information Processing Systems}, 2015.

\bibitem[Kalai and Vempala(2005)]{kalai2005efficient}
Adam Kalai and Santosh Vempala.
\newblock Efficient algorithms for online decision problems.
\newblock \emph{Journal of Computer and System Sciences}, 2005.

\bibitem[Kalyanakrishnan et~al.(2012)Kalyanakrishnan, Tewari, Auer, and
  Stone]{kalyanakrishnan2012pac}
Shivaram Kalyanakrishnan, Ambuj Tewari, Peter Auer, and Peter Stone.
\newblock Pac subset selection in stochastic multi-armed bandits.
\newblock In \emph{International Conference on Machine Learning}, 2012.

\bibitem[Karnin et~al.(2013)Karnin, Koren, and Somekh]{karnin2013almost}
Zohar Karnin, Tomer Koren, and Oren Somekh.
\newblock Almost optimal exploration in multi-armed bandits.
\newblock In \emph{International Conference on Machine Learning}, 2013.

\bibitem[Karnin(2016)]{karnin2016verification}
Zohar~S Karnin.
\newblock Verification based solution for structured mab problems.
\newblock In \emph{Advances in Neural Information Processing Systems}, 2016.

\bibitem[Kaufmann and Kalyanakrishnan(2013)]{kaufmann2013information}
Emilie Kaufmann and Shivaram Kalyanakrishnan.
\newblock Information complexity in bandit subset selection.
\newblock In \emph{Conference on Learning Theory}, 2013.

\bibitem[Kaufmann et~al.(2014)Kaufmann, Capp{\'e}, and
  Garivier]{kaufmann2014complexity}
Emilie Kaufmann, Olivier Capp{\'e}, and Aur{\'e}lien Garivier.
\newblock On the complexity of a/b testing.
\newblock In \emph{Conference on Learning Theory}, 2014.

\bibitem[Kolar et~al.(2011)Kolar, Balakrishnan, Rinaldo, and
  Singh]{kolar2011minimax}
Mladen Kolar, Sivaraman Balakrishnan, Alessandro Rinaldo, and Aarti Singh.
\newblock {Minimax localization of structural information in large noisy
  matrices}.
\newblock \emph{Advances in Neural Information Processing Systems}, 2011.

\bibitem[Krishnamurthy(2016)]{krishnamurthy2016minimax}
Akshay Krishnamurthy.
\newblock Minimax structured normal means inference.
\newblock \emph{IEEE International Symposium on Information Theory}, 2016.

\bibitem[Mannor and Tsitsiklis(2004)]{mannor2004sample}
Shie Mannor and John~N Tsitsiklis.
\newblock The sample complexity of exploration in the multi-armed bandit
  problem.
\newblock \emph{Journal of Machine Learning Research}, 2004.

\bibitem[Mossel et~al.(2014)Mossel, Neeman, and Sly]{mossel2014belief}
Elchanan Mossel, Joe Neeman, and Allan Sly.
\newblock Belief propagation, robust reconstruction and optimal recovery of
  block models.
\newblock In \emph{Conference on Learning Theory}, 2014.

\bibitem[Plotkin et~al.(1995)Plotkin, Shmoys, and Tardos]{plotkin1995fast}
Serge~A. Plotkin, David~B. Shmoys, and {\'E}va Tardos.
\newblock Fast approximation algorithms for fractional packing and covering
  problems.
\newblock \emph{Mathematics of Operations Research}, 1995.

\bibitem[Rakhlin and Sridharan(2016)]{rakhlin2016bistro}
Alexander Rakhlin and Karthik Sridharan.
\newblock Bistro: An efficient relaxation-based method for contextual bandits.
\newblock In \emph{International Conference on Machine Learning}, 2016.

\bibitem[Russo(2016)]{russo2016simple}
Daniel Russo.
\newblock Simple bayesian algorithms for best arm identification.
\newblock In \emph{Conference on Learning Theory}, 2016.

\bibitem[Simchowitz et~al.(2017)Simchowitz, Jamieson, and
  Recht]{simchowitz2017simulator}
Max Simchowitz, Kevin Jamieson, and Benjamin Recht.
\newblock The simulator: Understanding adaptive sampling in the
  moderate-confidence regime.
\newblock In \emph{Conference on Learning Theory}, 2017.

\bibitem[Syrgkanis et~al.(2016)Syrgkanis, Luo, Krishnamurthy, and
  Schapire]{syrgkanis2016improved}
Vasilis Syrgkanis, Haipeng Luo, Akshay Krishnamurthy, and Robert~E Schapire.
\newblock Improved regret bounds for oracle-based adversarial contextual
  bandits.
\newblock In \emph{Advances in Neural Information Processing Systems}, 2016.

\bibitem[Wang et~al.(2007)Wang, Gutell, and Miranker]{wang2007biclustering}
Shu Wang, Robin~R Gutell, and Daniel~P Miranker.
\newblock Biclustering as a method for rna local multiple sequence alignment.
\newblock \emph{Bioinformatics}, 2007.

\end{thebibliography}

\end{document}